\pdfoutput=1
\documentclass[onefignum,onetabnum,a4paper]{siamonline190516}
\usepackage[utf8]{inputenc}
\usepackage{t1enc}
\usepackage{amsfonts}
\usepackage{graphicx}
\usepackage{multirow}
\newsiamremark{remark}{Remark}
\headers{Learning via conjugate gradients and depth-varying neural ODEs}%
{G. Baravdish, G. Eilertsen, R. Jaroudi, B.~T. Johansson, L. Mal\'y, and J. Unger}
\title{Learning via nonlinear conjugate gradients and depth-varying neural ODEs}
\author{George Baravdish%
\thanks{Department of Science and Technology, Linköping University, SE-601\,74 Norrköping, Sweden\\
(\email{george.baravdish@liu.se},\hfill 
 \email{gabriel.eilertsen@liu.se},\hfill
 \email{rym.jaroudi@liu.se},\hfill
 \email{tomas.johansson@liu.se},\hfill
 \email{lukas.maly@liu.se},
 \email{jonas.unger@liu.se}).}
\and Gabriel Eilertsen\footnotemark[1]
\and Rym Jaroudi\footnotemark[1]
\and\linebreak B. Tomas Johansson\footnotemark[1]
\and Lukáš Malý\footnotemark[1]
\and Jonas Unger\footnotemark[1]}
\ifpdf
\hypersetup{
  pdftitle={Learning via nonlinear conjugate gradients and depth-varying neural ODEs},
  pdfauthor={G. Baravdish, G. Eilertsen, R. Jaroudi, B.~T. Johansson, L. Malý, and J. Unger}
}
\fi
\newcommand*{\Rbb}{{\mathbb{R}}}
\newcommand*{\Ccal}{{\mathcal{C}}}
\newcommand*{\Ecal}{{\mathcal{E}}}
\newcommand*{\Scal}{{\mathcal{S}}}
\newcommand*{\tran}{{\mathsf{T}}}
\newcommand*{\vek}{\boldsymbol}

\newcommand*{\vx}{\vek{x}}
\newcommand*{\vxi}{\vek{\xi}}
\newcommand*{\vy}{\vek{y}}
\newcommand*{\vz}{\vek{z}}
\newcommand*{\vb}{\vek{b}}
\newcommand*{\eps}{\varepsilon}
\newcommand*{\smax}{\mathop{\stackrel{\mathrm{soft}}{\max}}}
\DeclareMathOperator{\sech}{sech}
\begin{document}
\maketitle

\begin{abstract} 
The inverse problem of supervised reconstruction of depth-variable (time-dependent) parameters in a neural ordinary differential equation (NODE) is considered, that means finding the weights of a residual network with time continuous layers. The NODE is treated as an isolated entity describing the full network as opposed to earlier research, which embedded it between pre- and post-appended layers trained by conventional methods. The proposed parameter reconstruction is done for a general first order differential equation by minimizing a cost functional covering a variety of loss functions and penalty terms. A nonlinear conjugate gradient method (NCG) is derived for the minimization. Mathematical properties are stated for the differential equation and the cost functional. The adjoint problem needed is derived together with a sensitivity problem. The sensitivity problem can estimate changes in the network output under perturbation of the trained parameters. To preserve smoothness during the iterations the Sobolev gradient is calculated and incorporated. As a proof-of-concept, numerical results are included for a NODE and two synthetic datasets, and compared with standard gradient approaches (not based on NODEs). The results show that the proposed method works well for deep learning with infinite numbers of layers, and has built-in stability and smoothness.    
\end{abstract}

\begin{keywords}
Artificial neural networks, conjugate gradient method, deep learning, inverse problems, neural ordinary differential equations, Sobolev gradient
\end{keywords}

\begin{AMS}
68T07, 34H05, 49N45
\end{AMS}

\section{Introduction}
\label{sec:intro}
Machine learning in the form of supervised learning addresses the inverse problem of reconstructing or approximating a function from given data. Deep neural networks have been proven efficient for this purpose, as these can learn complex functions on large quantities of data. Neural networks are a form of parametrized computational architectures propagating information sequentially using linear parametric operators and certain nonlinearities. Research in this direction has been prolific and an overview cannot be given here; an introduction to neural networks with focus on inverse problems is in~\cite{arridge}, see also~\cite{higham}.

Residual Networks \cite{he2016deep,8100117,8099726} (ResNets) enable successful optimization of very deep neural network architectures with hundreds of layers. Its representational power has led to improvements in a range of high performance computer vision applications, such as classification, object detection, segmentation, etc. In the seminal paper~\cite{chen}, it was observed that the structure of a residual network is similar to the Euler discretization of an ordinary differential equation (ODE). This differential equation, termed a neural ordinary differential equation (NODE), takes the form
\begin{equation}
\label{eq:node}
\left\{
\begin{aligned}
   \vx'(t) & = F(t, \vx(t), \theta(t)), \quad t \in I := (0,T)\\
   \vx(0) & = \vx^0
\end{aligned}
\right.
\end{equation}
and corresponds to an infinite number of layers. Here, $F$ can describe activation of neurons in the network and depends on the depth-variable (i.e., time-dependent) parameters $\theta(t) \in \Rbb^M$, $t\in I$. Starting from the input layer $\vx^0$, the output layer $\vx(T)$ can be obtained by solving the ODE~\eqref{eq:node}. There are advantages of this time-continuous formulation. For example, in time-series and medical applications, it is important to have time as a continuous variable to make predictions. Using neural ODEs, it is possible to simulate forward in time as well as reversing to see what initial data that caused an observed effect. Furthermore, from theory for ODEs, qualitative statements about the output and parameter dependence can be derived, as well as enabling a trade-off between precision and speed in an already trained model. The apparent stability of methods based on NODEs is an attractive feature in real world applications, see for example~\cite{Owoyele_Pal}. Further advantages of neural ODEs are listed in~\cite[Secs.~2.1--2.2]{Fompeyrine}. A number of works for NODEs have been presented, we point to some in the remaining sections.

Supervised training entails fitting the output of the neural network to a given set of data points. In the continuous setting, an ODE is learned for function approximation. This means identifying the parameter function $\theta(t)$ representing weight matrices and biases. The output data consists of a (usually) large number of elements $\vy \in \Rbb^N$ and corresponding initial states $\vx^0$ connected by the condition $\vx(T)=\vy$ at some time-point (layer) $T>0$ with $\vx$ satisfying~\eqref{eq:node}. Identifying $\theta$ can be recast as minimizing a so-called cost functional
\begin{equation}
\label{eq:lfunctional}
J(\theta) = 
E(\vx, \vy, \theta)
\end{equation}
subject to $\vx(t)$ solving~\eqref{eq:node} and $\theta$ belonging to a suitable class of vector-valued (or matrix-valued) functions, e.g. $\theta \in \mathcal{C}([0,T]\to \Rbb^{M})$ or a Hilbert function space.
In~\cite{chen}, a numerical method is proposed for finding the parameter $\theta \in \Rbb^M$ introducing an adjoint equation to approximately compute a gradient of a specific cost functional.

A limitation of the original NODE in~\cite{chen} is that it considers a constant $\theta$ that does not vary with depth $t\in (0, T)$. This shortcoming has been pointed out in \cite{dissecting}, which also suggested an improved method to allow for depth-varying parameters $\theta(t)$. Previous work on NODEs have also only considered them in combination with conventional neural networks, where both the input and output of the NODE are processed by layers not within the ODE formulation. Thus, it has not been studied how NODEs perform in isolation. With an independent NODE architecture, the relation between input and output is directly described by the NODE, allowing for mathematical analysis, e.g., sensitivity analysis of network predictions.

We present a novel nonlinear conjugate gradient method (NCG) for minimizing a general cost functional~\eqref{eq:lfunctional} subject to~\eqref{eq:node} to find depth-varying (time-dependent) $\theta$, and consider~\eqref{eq:node} as an isolated network mapping from input to output for a broad class of functions~$F$. The cost functional covers commonly used ones in deep learning, consisting of loss functions and penalty terms. The adjoint and sensitivity problems are derived. The Fr\'{e}chet derivative of the cost functional is calculated in the $L^2$ sense and also with respect to a Sobolev norm, that is the Sobolev gradient is found. The derived framework has some distinct benefits. In particular, the sensitivity problem can be used when analysing error propagation in the learned parameters under changes in for example the initial condition. Moreover, as will be demonstrated in the numerical examples, the Sobolev gradient has a profound effect on the training since it generates smoother parameters, which in turn speed up the numerical solution of the NODEs.

As a proof-of-concept, we implement the proposed method and present numerical results for a NODE and two synthetic datasets. The results show promising behavior and compare well with standard SGD based methods, opening up for future use of isolated NODEs with depth-varying network weights. Employing the Sobolev gradient has a striking effect on the learned parameters, and appears not to have been tested earlier in deep learning. 

The present work builds on recent results~\cite{cao1,cao2} on inverse ill-posed coefficient and source identification problems for parabolic equations, see also~\cite{Alosaimi,Baravdish}. The inverse problems in those works are reformulated as the minimization of cost functionals, and a conjugate gradient method is employed to find a minimum. Note that parabolic equations can be seen as abstract ODEs taking values in suitable function spaces.

To summarize the contributions of the present work, we: 1) derive and analyse a novel method for identifying depth-varying parameters in NODEs by means of conjugate gradients for minimizing a general cost functional, 2) consider the NODE as a separate entity and do not combine it with other traditional architectures 3) calculate the Sobolev gradient of the cost functional and 4) include numerical results showing the actual reconstruction of such parameters on different datasets (in previous work numerical results are typically generated with constant parameters and placing the NODE within a traditional deep learning framework). 

We mention that inverse problems for ODEs have been studied previously, see for example~\cite{kunze,llibre}, but results do not cover the present setting of parameter estimation in models of deep learning. ODEs have also been used in the context of optimization algorithms (see~\cite{Brown} and references therein) but their usage to generate continuous depth-varying algorithms in deep learning is recent.

The paper is organized as follows. In Section~\ref{sec:ODEsolutions}, we recall some facts about the well-posedness of~\eqref{eq:node} for a general function $F$ together with a result on the differentiability of the solution with respect to parameters in the model, see Theorem~\ref{thm:existence_uniqueness} and Theorem~\ref{thm:differentiableDependence}. In Section~\ref{sec:minimization}, we introduce a general cost functional to be minimized, and a nonlinear conjugate gradient method for its minimization. The cost functional covers commonly used loss functions and penalty terms. For example, it can contain Tikhonov penalization of Sobolev type. The required adjoint and sensitivity problems are derived, see Theorem~\ref{thm:Eprim} and Theorem~\ref{thm:steplength}.  To keep smoothness of the updates during the iterations the Sobolev gradient is calculated, and its effect and expression are given in Remark~\ref{rem:smoothness} and~Corollary~\ref{cor:SobolevPairing}. At the end of Section~\ref{sec:minimization}, it is outlined that for certain choices of terms in the cost functional there exists a well-defined minimum and that it is possible by training to get arbitrarily close to the desired target, see Proposition~\ref{pro:minimizer}. Throughout the paper, the stated results are for a broad class of right-hand sides in the ODE~\eqref{eq:node}, and cover the case of NODEs in deep learning when the depth-varying trainable parameter $\theta$ is a vectorization of a weight matrix and a bias vector, and supplied with an activation function (admissible activation functions are discussed at the end of Section~\ref{sec:ODEsolutions}). Numerical results are given in Section~\ref{sec:numresults} for a NODE and two synthetic datasets known as a two moons dataset and two full circles dataset. It is demonstrated that with a rather direct implementation without optimizing the code the obtained results compare well with the ones from standard neural network toolboxes (which are not based on NODEs), and in particular generate stable classifications. The training is also performed using the Sobolev gradient improving smoothness of the sought parameters and speeding up the training. Some conclusions are drawn in Section~\ref{sec:conclusions}.

\section{Solutions and parameter dependence of ordinary differential equations}
\label{sec:ODEsolutions}
We start by recalling the following standard well-posedness result, which generalizes the Picard--Lindel\"of theorem and allow for possible discontinuities in the activation function $F$ along the layers (that is in time), keeping the notation $I=(0,T)$ with $T>0$.
\begin{theorem}[{cf.~\cite[Chap.~3]{Regan}}]
\label{thm:existence_uniqueness}
Consider the initial value problem
\[
\left\{
\begin{aligned}
   \vx'(t) & = F(t, \vx(t)), \quad t \in I,\\
   \vx(0) & = \vx^0
\end{aligned}
\right.
\]
with a general right-hand side $F: I\times \Rbb^N \rightarrow \Rbb^N$. Assume that $F$ is $L^p$-Carath\'{e}odory and $L^p$-Lipschitz in $\vx$, i.e., $F$ is measurable in $t$ and satisfies 
\[
  |F(t,\vx) - F(t,\vz)| \leq \alpha(t) |\vx-\vz|\quad\text{for every }\vx,\vz \in \Rbb^N\text{ and a.e.~}t\in I
\]
with $\alpha \in L^p(I)$. Then, there exists a unique solution $x \in W^{1,p}(I)$ to the given initial value problem.
\end{theorem}
The space $W^{1,p}(I)$ is the Sobolev space of order one, that is functions in $L^p$ having a distributional derivative in $L^p$. Such functions can, by the Sobolev embedding theorem, be redefined on a set of measure zero to be continuous on the interval $I$, and absolutely continuous on any compact subinterval of $I$, with a classical derivative in $L^p(I)$.

Theorem~\ref{thm:existence_uniqueness} yields that deep networks described by NODEs are well-defined and it is possible to run inference for a rather broad class of right-hand sides~$F$. However, higher regularity of~$F$ is required for the analysis of perturbations when training the network to find optimal parameters $\theta(t)$. Our method makes use of the following form of Peano's theorem that establishes differential dependence of a solution to the governing ODE on parameters.
\begin{theorem}[{cf.~\cite[Thm.~V.3.1]{hartman2002ordinary}}]
\label{thm:differentiableDependence}
Given a parameter $s \in \Rbb$, consider the initial value problem
\[
\left\{
\begin{aligned}
   \vx'(t) & = \widetilde{F}(t, \vx(t), s), \quad t \in I,\\
   \vx(0) & = \vx^0,
\end{aligned}
\right.
\]
where $\widetilde{F}: I\times \Rbb^N \times \Rbb \rightarrow \Rbb^N$ is continuous. Assume also that $\widetilde{F}$ is continuously differentiable with respect to $\vx$ and $s$. Then, the unique solution $\vx = \vx(t; s)$ to the given initial value problem is continuously differentiable with respect to the parameter $s$\/~{\rm(}and the initial condition $\vx^0$\/{\rm)}. Moreover, $\vxi(t) := D_s \vx(t; s)$ satisfies the linear equation
\[
\left\{
\begin{aligned}
   \vxi'(t) & = D_{\vx} \widetilde{F}(t, \vx(t; s), s) \vxi(t) + D_s \widetilde{F}(t,\vx(t; s), s), \quad t \in I,\\
   \vxi(0) & = 0.
\end{aligned}
\right.
\]
\end{theorem}
For weaker assumptions on $\widetilde{F}$ under which the solution $\vx(t; s)$ is differentiable with respect to the parameter $s$, see~\cite{Ursescu} with an overview of results given in~\cite{Khan}.

In our numerical experiments, we will use a special structure of the right-hand side of the NODE~\eqref{eq:node}, namely,
\begin{equation}
\label{eq:F}
  F(t, \vx, \theta) = \sigma(W\vx + \vb),
\end{equation}
with the depth-varying (time-dependent) trainable parameter $\theta \in \Rbb^{N^2 + N}$ being a vectorization of a weight matrix $W \in \Rbb^{N\times N}$ and a bias vector $\vb \in \Rbb^N$. Here, $\sigma$ is an activation function applied component-wise, e.g., ReLU (rectified linear unit) defined by $r \mapsto \max\{0, r\}$. 
It is clear that for $F$ as in \eqref{eq:F},  Theorem~\ref{thm:existence_uniqueness} applies whenever $\sigma$ is Lipschitz continuous (i.e., including ReLU) and $W, \vb \in L^p(I)$. 
The differential dependence on parameters in Theorem~\ref{thm:differentiableDependence} applies to \eqref{eq:F} when the trainable parameters $\theta$ are perturbed by $s\mkern 1mu \eta$ as in $\widetilde{F}(t, \vx, s) = F(t, \vx, \theta(t) + s \eta(t))$, provided that the activation function $\sigma$ is continuously differentiable and the parameter functions $\theta$ and $\eta$ are continuous. Even though this requirement excludes ReLU or Leaky ReLU to be used as $\sigma$, one may use a $\Ccal^1$-smooth regularization instead, e.g., the softplus or ELU (exponential linear unit, with parameter $1$).
Other commonly used activation functions include the sigmoid and the hyperbolic tangent (TanH), both being smoother than ReLU (thus the above results still apply) but having known drawbacks in deep learning neural networks (such as the ``vanishing gradient problem'').

\section{Minimization of the cost functional}
\label{sec:minimization}
Consider a neural network described by the NODE \eqref{eq:node}. Given a batch of $K$ data points, each with input $\vx_k^0 \in \Rbb^N$ and desired output $\vy_k \in \Rbb^N$, we seek a parameter $\theta: [0, T] \to \Rbb^M$ that minimizes the cost functional
\begin{equation}
  \label{eq:E}
  E(\theta) = \frac{1}{K} \sum_{k=1}^K \Biggl(L\bigl(\vx_k(T), \vy_k\bigr) + \int_0^T \ell\bigl(\vx_k(t), \vy_k\bigr)\,dt\Biggr) 
  + \int_0^T \Phi\bigl(\theta(t), \theta'(t)\bigr)\,dt,
\end{equation}
subject to the NODEs
\begin{equation}
\label{eq:kNODE}
	\left\{
	\begin{aligned}
		 \vx_k'(t) & = F(t, \vx_k(t), \theta(t)), \quad t \in I,\\
		 \vx_k(0) & = \vx_k^0
	\end{aligned}
	\right.  
\end{equation}
with $k=1,2,\ldots, K$. Here, $L, \ell: \Rbb^N\times \Rbb^N \to \Rbb^+$ are loss functions that measure the ``distance'' between the desired output $\vy_k$ and $\vx_k(t)$, and possibly the magnitude of $\vx_k(t)$, while $\Phi: \Rbb^M \times \Rbb^M \to \Rbb$ measures the magnitude of the parameters and/or their derivative.
Note that one may very well choose $\ell$ and $\Phi$ to be identically zero provided that $L(\vx, \vy) > 0$ whenever $\vx \ne \vy$. We assume that $L$, $\ell$, and $\Phi$ are differentiable. For stability, the function $\Phi$ can contain a multiplicative factor $\lambda>0$ acting as a regularization parameter.

A standard example of $L$ or $\ell$ is the squared $\ell^2$-distance, i.e., $L(\vx, \vy) = \frac{1}{2} \| \vx - \vy \|_{\ell^2}^2$. Alternatively, if $\vy$ is a probability vector, then $L(\vx, \vy)$ can be defined as cross-entropy after $\vx$ has been normalized to a probability vector using, e.g., softmax. 

The final term in the right-hand side of~\eqref{eq:E} can act as a Tikhonov penalty term. It can be useful in practical applications to not only penalize the parameter values $\theta$ but also the derivative $\theta'$, which motivates allowing $\Phi$ to depend also on $\theta'$.
\subsection*{A nonlinear conjugate gradient method for minimizing the cost functional}
\addcontentsline{toc}{subsection}{A nonlinear conjugate gradient method for minimizing the cost functional}
We propose an iterative method based on the adjoint problem and the conjugate gradient method in order to search for a minimizer $\theta$ of $E(\theta)$ in~\eqref{eq:E}. Given an initial parameter $\theta_0: I \to \Rbb^M$, we find the direction of steepest descent for $E$ via the \emph{adjoint problem} stated in \eqref{eq:AP} below, see Theorem~\ref{thm:Eprim}. The descent direction is thus $\eta_0 = -E'(\theta_0)$. Having solved what is known as a \emph{sensitivity problem} (see \eqref{eq:SP} below), we determine the learning rate $\beta_0>0$ and update the parameter by setting $\theta_1 = \theta_0 + \beta_0 \eta_0$.

In $(j+1)^{\text{st}}$ iteration, the direction of steepest descent is $-E'(\theta_{j})$. The Fletcher--Reeves conjugate gradient coefficient $\gamma_j$ is computed by
\[
    \gamma_j = \frac{\| E'(\theta_{j}) \|^2}{\| E'(\theta_{j-1}) \|^2}
\]
and we set the descent direction to $\eta_j = -E'(\theta_{j}) + \gamma_j \eta_{j-1}$. The norm in the expression for $\gamma_j$ is either the norm of $L^2(I)$ or the Sobolev space $W^{1,2}(I)$; alternative expressions for the conjugate gradient coefficients are listed in~\cite[p.~126]{Andrei}. The learning rate $\beta_j>0$ is then determined via the sensitivity problem \eqref{eq:SP} and we update the sought parameter by setting $\theta_{j+1} = \theta_j + \beta_j \eta_j$.

To calculate the conjugate gradient coefficient $\gamma_j$, we start by finding $E'(\theta)$ in the classical sense, and then extend the result to generate the corresponding Fr\'{e}chet derivative in the function space $L^2(I)$ respective the Sobolev space $W^{1,2}(I)$.
\begin{theorem}
\label{thm:Eprim}
Assume that $L$, $\ell$, and\, $\Phi$ in \eqref{eq:E} as well as $F$ and $\theta$ in \eqref{eq:kNODE} are continuously differentiable. Let $\eta: I \to \Rbb^M$ be continuously differentiable. Then, the directional derivative of the cost functional $E$ in \eqref{eq:E} at $\theta$ along $\eta$ is given by
\begin{align}
 \label{eq:E'}
   E'_\eta (\theta) & = \int_0^T \biggl( D_\theta \Phi\bigl(\theta(t), \theta'(t)\bigr) + \frac{1}{K} \sum_{k=1}^K \lambda_k(t)^\tran D_\theta F\bigl(t, \vx_k(t), \theta(t)\bigr)\biggr) \eta(t)\,dt \\
   \notag
   & \quad + \int_0^T D_{\theta'} \Phi\bigl(\theta(t), \theta'(t)\bigr) \eta'(t)\,dt,
\end{align}
where $\lambda_k: I \to \Rbb^N$ is the solution to the \emph{adjoint problem}
\begin{equation}
\label{eq:AP}
\left\{
\begin{aligned}
\lambda_k'(t) & = - D_{\vx}F(t, \vx_k(t), \theta(t))^\tran \lambda_k(t) - D_{\vx} \ell(\vx_k(t), \vy_k)^\tran, \quad t \in I = (0,T), \\
\lambda_k(T) & = D_{\vx} L(\vx_k(T), \vy_k)^\tran.
\end{aligned}
\right.
\end{equation}

If $\Phi$ does not depend on the second variable, i.e., $\Phi(\theta(t), \theta'(t)) =\tilde{\Phi}(\theta(t))$, then it suffices to assume that $\theta$ and $\eta$ are merely continuous instead of continuously differentiable.
\end{theorem}
\begin{proof}
Since the element $\vx_k$ satisfies \eqref{eq:kNODE}, we may add a Lagrange relaxation term to $E(\theta)$ and thus define
\begin{align*}
  \Ecal(\theta) &= \frac{1}{K} \sum_{k=1}^K \Biggl(L\bigl(\vx_k(T), \vy_k\bigr) + \int_0^T \ell\bigl(\vx_k(t), \vy_k\bigr)\,dt\Biggr) + \int_0^T \Phi\bigl(\theta(t), \theta'(t)\bigr)\,dt \\
	& \quad + \frac{1}{K} \sum_{k=1}^K \int_0^T \bigl\langle \lambda_k(t), F\bigl(t, \vx_k(t), \theta(t)\bigl) - \vx'_k(t)\bigr\rangle\,dt,
\end{align*}
where $\lambda_k: I \to \Rbb^N$ is yet to be determined. The last integral above contains the standard inner product in $\Rbb^N$ defined by $\langle\vx, \vy \rangle = \vy^\tran \vx$. Given the parameters $\theta, \eta: I \to \Rbb^M$, we shall find the directional derivative
\begin{equation}
\label{eq:dElim}
  \Ecal'_\eta(\theta) = \lim_{\eps \to 0} \frac{\Ecal(\theta + \eps \eta) - \Ecal(\theta)}{\eps}\,.
\end{equation}
Let $\widetilde{F}(t, \vx, \eps) = F(t, \vx, \theta(t) + \eps \eta(t))$. If $\vx_k$ is a solution to \eqref{eq:kNODE}, then it is also a solution to the initial value problem
\[
  \left\{
	\begin{aligned}
	\vx_k'(t) & = \widetilde{F}(t, \vx_k(t), \eps), \quad t \in I,\\
	\vx_k(0) & = \vx_k^0
	\end{aligned}
	\right.
\]
with $\eps = 0$. By Theorem~\ref{thm:differentiableDependence}, $\vx_k$ is differentiable with respect to the parameter $\eps$. In particular, when $\theta$ is perturbed by $\eps \eta$, then $\vx_k = \vx_k(t; \theta + \eps \eta) = \vx_k(t; \theta) + \eps \vxi_k(t) + o(\eps)$,
where $\vxi_k$ satisfies
\[
\left\{
\begin{aligned}
   \vxi_k'(t) & = D_{\vx} \widetilde{F}(t, \vx_k(t; \theta), 0) \vxi_k(t) + D_\eps \widetilde{F}(t,\vx_k(t; \theta), 0), \quad t \in I,\\
   \vxi_k(0) & = 0.
\end{aligned}
\right.
\]
Here, we used the little-o notation, meaning that $\frac{o(\eps)}{\eps} \to 0$ as $\eps \to 0$. Expressed in terms of $F$, the perturbation of the solution $\vxi_k$ satisfies the linear differential equation
\[
\left\{
\begin{aligned}
   \vxi_k'(t) & = D_{\vx} F(t, \vx_k(t; \theta), \theta(t)) \vxi_k(t) + D_\theta F(t,\vx_k(t; \theta), \theta(t)) \eta(t), \quad t \in I,\\
   \vxi_k(0) & = 0.
\end{aligned}
\right.
\]

The difference quotient in \eqref{eq:dElim} can be expressed as
\begin{align*}
  \frac{\Ecal(\theta + \eps \eta) - \Ecal(\theta)}{\eps} & = \frac{1}{K} \sum_{k=1}^K \frac{L(\vx_k(T; \theta) + \eps \vxi_k(T) + o(\eps), \vy_k) - L(\vx_k(T;\theta), \vy_k)}{\eps} \\
	& \quad + \frac{1}{K} \sum_{k=1}^K \int_0^T \frac{\ell(\vx_k(t; \theta) + \eps \vxi_k(t) + o(\eps), \vy_k) - \ell(\vx_k(t; \theta), \vy_k)}{\eps}\,dt \\
	& \quad + \int_0^T \frac{\Phi\bigl(\theta(t) + \eps \eta(t), \theta'(t) + \eps \eta'(t)\bigr) - \Phi\bigl(\theta(t), \theta'(t)\bigr)}{\eps}\,dt \\
	& \quad + \frac{1}{K} \sum_{k=1}^K \int_0^T \left\langle \lambda_k(t), \frac{F(t, \vx_k(t; \theta + \eps \eta), \theta(t) + \eps \eta(t)) - F(t, \vx_k(t; \theta), \theta(t))}{\eps}\right.\\
	& \qquad\qquad\qquad\qquad\qquad - \left.\frac{\vx_k'(t; \theta + \eps \eta) - \vx_k'(t; \theta)}{\eps}\right\rangle\,dt.
\end{align*}
Taking the limit as $\eps \to 0$ yields that
\begin{align*}
  \Ecal'_\eta(\theta) & = \frac{1}{K} \sum_{k=1}^K \Biggl( D_{\vx} L(\vx_k(T;\theta), \vy_k) \vxi_k(T) + \int_0^T D_{\vx} \ell(\vx_k(t;\theta), \vy_k) \vxi_k(t)\,dt \Biggr) \\
	& \quad + \int_0^T \Bigl(D_\theta \Phi\bigl(\theta(t), \theta'(t)\bigr)\, \eta(t) + D_{\theta'} \Phi\bigl(\theta(t), \theta'(t)\bigr)\, \eta'(t)\Bigr)\,dt \\
	& \quad + \frac{1}{K}\sum_{k=1}^K \int_0^T \Bigl\langle \lambda_k(t), D_{\vx} F\bigl(t, \vx_k(t; \theta), \theta(t)\bigr) \vxi_k(t) + D_\theta F\bigl(t,\vx_k(t;\theta), \theta(t)\bigr) \eta(t)\Bigr\rangle\,dt\\
	& \quad - \frac{1}{K}\sum_{k=1}^K \int_0^T \Bigl\langle \lambda_k(t), \vxi_k'(t)\Bigr\rangle\,dt.
\end{align*}
Integration by parts in the last integral gives
\[
\int_0^T \Bigl\langle \lambda_k(t), \vxi_k'(t)\Bigr\rangle\,dt = \Bigl\langle \lambda_k(T), \vxi_k(T)\Bigr\rangle - \Bigl\langle \lambda_k(0), \vxi_k(0)\Bigr\rangle - \sum_{k=1}^K \int_0^T \Bigl\langle \lambda_k'(t), \vxi_k(t)\Bigr\rangle\,dt,
\]
where the middle term is zero as $\vxi_k(0) = 0$ by Theorem~\ref{thm:differentiableDependence}. Rearranging terms in $\Ecal'_\eta(\theta)$, we obtain
\begin{align*}
  \Ecal'_\eta(\theta) & = \frac{1}{K} \sum_{k=1}^K \Bigl\langle D_{\vx} L(\vx_k(T;\theta), \vy_k)^\tran - \lambda_k(T), \vxi_k(T)\Bigr\rangle \\
	& \quad + \frac{1}{K}\sum_{k=1}^K \int_0^T \Bigl\langle D_{\vx} \ell(\vx_k(t;\theta), \vy_k)^\tran + D_{\vx} F\bigl(t, \vx_k(t; \theta), \theta(t)\bigr)^\tran \lambda_k(t) + \lambda_k'(t), \vxi_k(t) \Bigr\rangle\,dt\\
	& \quad + \int_0^T \left\langle D_\theta \Phi\bigl(\theta(t), \theta'(t)\bigr)^\tran + \frac{1}{K}\sum_{k=1}^K D_\theta F\bigl(t,\vx_k(t;\theta), \theta(t)\bigr)^\tran \lambda_k(t),  \eta(t)\right\rangle\,dt \\
	& \quad + \int_0^T \left\langle D_{\theta'} \Phi\bigl(\theta(t), \theta'(t)\bigr)^\tran, \eta'(t)\right\rangle\,dt.
\end{align*}
If $\lambda$ is a solution of the adjoint problem \eqref{eq:AP}, then
\begin{align*}
  \Ecal'_\eta(\theta) & = \int_0^T \left\langle D_\theta \Phi(\theta(t), \theta'(t)) + \frac{1}{K}\sum_{k=1}^K D_\theta F\bigl(t,\vx_k(t;\theta), \theta(t)\bigr)^\tran \lambda_k(t),  \eta(t)\right\rangle\,dt \\
  & \quad + \int_0^T \left\langle D_{\theta'} \Phi\bigl(\theta(t), \theta'(t)\bigr)^\tran, \eta'(t)\right\rangle\,dt.
\end{align*}
Since $\Ecal(\theta) = E(\theta)$ for all parameter functions $\theta: I \to \Rbb^M$, we have hereby proven \eqref{eq:E'}.
\end{proof}
\begin{corollary}
\label{cor:L2-derivative}
In addition to the assumptions of Theorem~\ref{thm:Eprim}, suppose that $\Phi$ does not depend on the second variable, i.e., $\Phi(\theta(t), \theta'(t)) = \tilde{\Phi}(\theta(t))$.
Then, the Fréchet derivative of the cost functional $E(\theta)$ in \eqref{eq:E} with respect to the dual pairing in $L^2(I \to \Rbb^M)$ is given by
\begin{equation}
\label{eq:E'-L2}
  E'(\theta) = \nabla \tilde\Phi(\theta(t)) + \frac{1}{K} \sum_{k=1}^K \lambda_k(t)^\tran D_\theta F(t, \vx_k(t), \theta(t)).
\end{equation}
\end{corollary}

The gradient $E'(\theta)$ is used when updating elements in the iterative process. Calculating this gradient in $L^2$ can lead to loss of smoothness in the generated elements. If it is {\it a priori} known that updates should be smooth, it is usually better to calculate $E'(\theta)$ with respect to for example the Sobolev space $W^{1,2}$. This is related to the concept of a Sobolev gradient, see~\cite[Chap.~8.5]{Alifanov} for calculations of such gradients in inverse heat transfer problems; an overview of works employing Sobolev gradients are given in~\cite[Sec.~5.2]{Jin}, see also~\cite{Neuberger}. The following result can be used to replace the pairing in $L^2$ with a pairing in $W^{1,2}$ when finding $E'(\theta)$.
\begin{theorem}
\label{thm:L2-W12_convert}
Given $u \in L^2(0,T)$, let
\begin{equation}
\label{eq:L2-W12_representative}
  v(t) = \frac{\cosh(T-t)}{\sinh T} \int_0^t u(s) \cosh(s)\,ds + \frac{\cosh t}{\sinh T} \int_t^T u(s) \cosh(T-s)\,ds, \quad t\in [0, T].
\end{equation}
Then,
\begin{equation}
\label{eq:L2-W12}  
  \int_0^T u(s) \phi(s)\,ds = \int_0^T \bigl(v(s) \phi(s) + v'(s) \phi'(s)\bigr)\,ds
\end{equation}
for every $\phi \in W^{1,2}(0, T)$.
\end{theorem}
\begin{proof}
The inner product of $L^2(0,T)$ defines continuous linear functionals on the space $W^{1,2}(0,T)$, hence the Riesz representation theorem for Hilbert spaces implies that there is an element $v\in W^{1,2}(0, T)$ that represents $u\in L^2(0,T)$ as a dual element with respect to the inner product of $W^{1,2}(0, T)$. Thus, the existence of $v$ satisfying~\eqref{eq:L2-W12} is clear; we shall show that $v$ can be taken in the form~\eqref{eq:L2-W12_representative}.
  
Assume that $u \in L^2(0,T)$ is continuous and let $v$ be defined by \eqref{eq:L2-W12_representative}. Then, a direct calculation shows that $v \in \Ccal^2(0,T) \cap \Ccal^{0,1}[0,T]$ is a solution to the boundary value problem
\begin{equation}
	\label{eq:BVP_W12-repre}
	\left\{
		\begin{aligned}
			 v''(t) - v(t) & = -u(t), \quad t \in (0, T), \\
			 v'(0) &= 0, \\
			 v'(T) & = 0.
		\end{aligned}
	\right.
\end{equation}
Thus, integration by parts yields that
\begin{align*}
	 \int_0^T \bigl(v(s) \phi(s) + v'(s) \phi'(s)\bigr)\,ds & = v'(T) \phi(T) - v'(0) \phi(0) - \int_0^T \bigl(v''(s) - v(s)\bigr) \phi(s)\,ds \\
	 & = \int_0^T u(s) \phi(s)\,ds
\end{align*}
for every $\phi \in W^{1,2}(0,T)$.

It remains to show that formula \eqref{eq:L2-W12} also holds true in the case when $v$ is defined by \eqref{eq:L2-W12_representative} with $u \in L^2(0, T)$ not continuous. By the density of continuous functions in $L^2(0,T)$, we can find a sequence $\{u_j\}_{j=1}^\infty \subset \Ccal[0,T]$ such that $u_j \to u$ in $L^2(0, T)$ and pointwise a.e. in $(0, T)$. Let $v_j \in \Ccal^2(0,T)  \cap \Ccal^{0,1}[0,T] \subset W^{1,2}(0,T)$ be the corresponding Sobolev representatives given by~\eqref{eq:L2-W12_representative}. Direct computation renders
\[
	v_j'(t) = \frac{\sinh(t-T)}{\sinh T} \int_0^t u_j(s) \cosh(s)\,ds + \frac{\sinh t}{\sinh T} \int_t^T u_j(s) \cosh(T-s)\,ds, \quad t\in [0, T].
\]
For $j\neq k$, we obtain by the triangle inequality that
\begin{align*}
	|v_j(t) - v_k(t)| & \le \frac{\cosh(T-t)}{\sinh T} \int_0^t |u_j(s)-u_k(s)| \cosh(s)\,ds \\
	& \quad + \frac{\cosh t}{\sinh T} \int_t^T |u_j(s)-u_k(s)| \cosh(T-s)\,ds \\
	& \le C_T \|u_j - u_k\|_{L^2(0,T)}
\end{align*}
and
\begin{align*}
	|v_j'(t) - v_k'(t)| & \le \frac{\sinh(T-t)}{\sinh T} \int_0^t |u_j(s)-u_k(s)| \cosh(s)\,ds \\
	& \quad + \frac{\sinh t}{\sinh T} \int_t^T |u_j(s)-u_k(s)| \cosh(T-s)\,ds \\
	& \le \widetilde{C}_T \|u_j - u_k\|_{L^2(0,T)}
\end{align*}
for every $t \in [0, T]$. Since $\{u_j\}_{j=1}^\infty$ is a Cauchy sequence in $L^2(0,T)$ this implies that $\{v_j\}_{j=1}^\infty$ is a Cauchy sequence in $W^{1,2}(0,T)$. Thus, $v_j \to v$ in $W^{1,2}(0,T)$ and \eqref{eq:L2-W12} remains valid also for the limit function $v$. 
\end{proof}
\begin{remark}
\label{rem:smoothness}
The Sobolev gradient affects the smoothness of $E'(\theta)$. This can be seen since it follows from \eqref{eq:L2-W12_representative} that when $u \in L^2(0,T)$ then $v \in W^{2,2}(0,T) \subset \Ccal^1[0, T]$ and $v$ is $k$-Lipschitz continuous with $k \le \|u\|_{L^1(0,T)}$. Moreover, a simple modification of the proof allows us to state the theorem above in a more general form. Namely, if $u \in L^p(0,T)$ with $p \in [1, \infty]$, then the function $v$ as defined in \eqref{eq:L2-W12_representative} lies in $W^{2,p}(0,T) \subset \Ccal^{1, 1/p'}[0,T]$ and it satisfies \eqref{eq:L2-W12} for every $\phi \in W^{1,p'}(0,T)$, where $p'$ is the H\"{o}lder conjugate exponent. If also $u \in W^{m,p}(0,T)$ for some $m\ge 1$, then $v \in W^{m+2, p}(0,T)$.
\end{remark}

We shall then give an explicit expression for $E'(\theta)$ when the pairing is with respect to the Sobolev space $W^{1,2}$, and for this we need the next result.
\begin{corollary}
\label{cor:L2'-W12_convert}
Given $u \in L^2(0,T)$, let $\displaystyle U(t) = \int_0^t u(s)\,ds$ and
\[
	v(t) = \frac{\cosh(T-t)}{\sinh T} \int_0^t U(s) \cosh(s)\,ds + \frac{\cosh t}{\sinh T} \int_t^T U(s) \cosh(T-s)\,ds, \quad t\in [0, T].
\]
Then  
\begin{equation}
	\label{eq:L2'-W12}  
		\int_0^T u(s) \phi'(s)\,ds = \int_0^T \Bigl(\bigl(U(s)-v(s)\bigr) \phi(s) + \bigl(U'(s)-v'(s)\bigr) \phi'(s)\Bigr)\,ds
\end{equation}
for every $\phi \in W^{1,2}(0, T)$.
\end{corollary}
\begin{proof}
As the function $U$ is a Lebesgue primitive to $u \in L^2(0,T) \subset L^1(0,T)$, we have
\begin{align*}
	\int_0^T u(s) \phi'(s)\,ds & = \int_0^T \bigl(U(s) \phi(s) + u(s) \phi'(s)\bigr)\,ds - \int_0^T U(s) \phi(s)\,ds \\
	& = \int_0^T \Bigl(\bigl(U(s)-v(s)\bigr) \phi(s) + \bigl(u(s)-v'(s)\bigr) \phi'(s)\Bigr)\,ds
\end{align*}
by Theorem~\ref{thm:L2-W12_convert}.
\end{proof}
We can then give the expression for $E'(\theta)$.
\begin{corollary}
\label{cor:SobolevPairing}
Under the assumptions of Theorem~\ref{thm:Eprim}, there is a $v \in W^{1,2}(I \to \Rbb^M)$ that represents the Fréchet derivative of the cost functional $E(\theta)$ in   \eqref{eq:E} with respect to the dual pairing in the Sobolev space $W^{1,2}(I \to \Rbb^M)$.
\end{corollary}
\begin{proof}
It follows from \eqref{eq:E'} that the directional derivative of $E$ at $\theta$ along $\eta$ can be expressed by
\[
	E'_\eta(\theta) = \sum_{j=1}^M \int_0^T \bigl(u_j(t) \eta_j(t) + \tilde{u}_j(t) \eta_j'(t)\bigr)\,dt,
\]
where $u_j, \tilde{u}_j \in L^2(0,T)$ and $\eta = (\eta_1, \eta_2, \ldots, \eta_M)^\tran \in W^{1,2}(I \to \Rbb^M)$. Theorem~\ref{thm:L2-W12_convert} and Corollary~\ref{cor:L2'-W12_convert} imply that there are functions $v_j, \tilde{v}_j \in W^{1,2}(0,T)$ such that
\[
	\int_0^T u_j(t) \eta_j(t)\,dt = \int_0^T \bigl(v_j(t) \eta_j(t) + v_j'(t) \eta_j'(t)\bigr)\,dt
\]
and
\[
	\int_0^T \tilde{u}_j(t) \eta'_j(t)\,dt = \int_0^T \bigl(\tilde{v}_j(t) \eta_j(t) + \tilde{v}_j'(t) \eta_j'(t)\bigr)\,dt.
\]
Thus, $v = (v_1 + \tilde{v}_1, v_2 + \tilde{v}_2, \ldots, v_M + \tilde{v}_M)^\tran$ is the sought representative of the Fréchet derivative of $E(\theta)$.
\end{proof}

In the iterations of NCG, we determine the optimal learning rate $\beta$ by minimizing an auxiliary functional $\widetilde{E}(\beta)$, which utilizes the solution of a so-called \emph{sensitivity problem}. It takes the following form.
\begin{theorem}
\label{thm:steplength}
Assume that $L$, $\ell$, $\Phi$, $F$ and $\theta$ are as in Theorem~\ref{thm:Eprim}. Assume also that $\vx_k$ are solutions to the NODEs \eqref{eq:kNODE}. Let $\eta: I \to \Rbb^M$ be continuously differentiable. Then,
\begin{align*}
	E(\theta + \beta \eta) & = \frac{1}{K} \sum_{k=1}^K \Biggl(L\bigl(\vx_k(T) + \beta \vxi_k(T), \vy_k\bigr) + \int_0^T \ell\bigl(\vx_k(t) + \beta \vxi_k(t), \vy_k\bigr)\,dt\Biggr)  + o(\beta) \\
& \qquad + \int_0^T \Phi\bigl(\theta(t)+ \beta \eta(t), \theta'(t)+ \beta \eta'(t)\bigr)\,dt =: \widetilde{E}(\beta) + o(\beta),
\end{align*}
and
\begin{align*}
	\widetilde{E}'(\beta) & = \frac{1}{K} \sum_{k=1}^K \Biggl(D_{\vx}L\bigl(\vx_k(T) + \beta \vxi_k(T), \vy_k\bigr)\vxi_k(T) + \int_0^T D_{\vx}\ell\bigl(\vx_k(t) + \beta \vxi_k(t), \vy_k\bigr)\vxi_k(t)\,dt\Biggr)\\
& \qquad + \int_0^T D_{\theta}\Phi\bigl(\theta(t)+ \beta \eta(t), \theta'(t)+ \beta \eta'(t)\bigr)\eta(t)\,dt\\
& \qquad + \int_0^T D_{\theta'}\Phi\bigl(\theta(t)+ \beta \eta(t), \theta'(t)+ \beta \eta'(t)\bigr)\eta'(t)\,dt,
\end{align*}
where $\vxi_k: I\to \Rbb^N$ are solutions to the \emph{sensitivity problem}
\begin{equation}
\label{eq:SP}
	\left\{
	\begin{aligned}
	 \vxi_k'(t) & = D_{\vx} F(t, \vx_k(t), \theta(t)) \vxi_k(t) + D_\theta F(t,\vx_k(t), \theta(t)) \eta(t), \quad t \in I,\\
	 \vxi_k(0) & = 0.
	\end{aligned}
	\right.
\end{equation}
\end{theorem}
\begin{proof}
Similarly as in the proof of Theorem~\ref{thm:Eprim}, we may apply the differentiable dependence on parameters stated in Theorem~\ref{thm:differentiableDependence} to see that the solution to the NODEs \eqref{eq:kNODE} with parameter $\theta + \beta \eta$ can be approximated by the solutions to the NODEs with parameter $\theta$ and the sensitivity problem \eqref{eq:SP} such that
\begin{equation}
\label{eq:expansion_sol}
  \vx_k(t; \theta + \beta \eta) = \vx_k(t; \theta) + \beta \vxi_k(t) + o(\beta).
\end{equation}
Thus, 
\begin{equation}
\label{eq:expansion_cost}
E(\theta + \beta \eta) = \widetilde{E}(\beta) + o(\beta)
\end{equation}
as both $L$ and $\ell$ are Lipschitz continuous on compact sets. The derivative $\widetilde{E}'(\beta)$ is obtained by direct differentiation. Note that the integrands have bounded derivatives on compact intervals, hence one may differentiate under the integral sign.
\end{proof}

We remark that from~\eqref{eq:expansion_sol} and~\eqref{eq:expansion_cost}, it follows that the sensitivity problem can be used when computing effects of changes in the learned parameters on inference accuracy and when investigating error propagation in the classification.

We have then explicit expressions for the quantities needed to realize the steps of the proposed NCG given at the beginning of this section. We shall not go into details on convergence of this method, the reader can find results on convergence of nonlinear conjugate gradient methods in the recent work~\cite[Chap.~3]{Andrei}, where also an overview of different nonlinear conjugate gradient methods are given (\cite[pp.~41--43]{Andrei}). For convergence, it is of course important to know that there indeed exists a minimum to~\eqref{eq:E}; existence is discussed in the next section.
\subsection*{Minimizer of the cost functional~(\ref{eq:E})}
\addcontentsline{toc}{subsection}{Minimizer of the cost functional~(\ref{eq:E})}
We briefly discuss existence and properties of a minimizer of the above cost functional~\eqref{eq:E}. In~\cite{Zuazua}, classes of cost functionals are investigated, for example, \eqref{eq:E} with $\ell=0$, $K\le N$,  and the term involving $\Phi$ in~\eqref{eq:E} being equal to 
\[
  \lambda\|\theta\|_{H^k(I \to \Rbb^{N^2+N})}^2
\]
for $k=0, 1$, with $H^k$ denoting the standard Lebesgue/Sobolev space of order $k$ endowed with an inner product, i.e., $L^2$ or $W^{k,2}$. Existence of a minimizer to that class of cost functionals with $F$ as in~\eqref{eq:F} is discussed and guaranteed when $k=1$, see~\cite[Rem.~1]{Zuazua}. Note that the cost functional we consider in our paper covers the case when $k=0$ ($L^2$-regularization) as well as $k=1$ ($H^1$-regularization) in~\cite{Zuazua}. Furthermore, under the basic assumption that 
the desired output vectors after activation, i.e., $\{\sigma(\vy_1), \ldots, \sigma(\vy_K)\} \subset \Rbb^N$, where $\sigma$ is a smooth activation function, form a linearly independent set, it is shown in \cite[Thm.~5.1]{Zuazua} that a minimizer approaches the zero training regime meaning that it meets the desired output.
Hence, \cite{Zuazua}~implies the following result.
\begin{proposition}
\label{pro:minimizer}
Assume that $N\ge K$. Then, there are cost functionals of the form\/~{\rm \eqref{eq:E}} for which, by adjusting parameters, there exists a minimizer making the training error arbitrarily close to zero, that is close to the minimum value of the loss term.
\end{proposition}
Putting an upper bound on the distance between the initial data and the desired output, the training error can be zero (\cite[Thm.~5.1]{Zuazua}). 

In the case when $L$ is based on a norm, such as $L(\vx, \vy) = \frac{1}{2} \| \vx - \vy \|_{\ell^2}^2$, conditions on $\Phi$ to guarantee the existence of a minimum to \eqref{eq:E} are given in~\cite[Thm.~4.1]{Schuster2012} and~\cite[Prop.~2]{Hofmann}.

Analysis of the minimization of cost functionals in the case of discretized ODEs in deep learning is given in~\cite{benning}.

For an inverse problem, uniqueness of a solution is usually desired. We remark that in applications of neural networks, the actual numerical values of the weights themselves have no apparent physical meaning. It is only requested to find a set of weights that can meet the desired training output. Moreover, minimizing the loss term is usually a highly non-convex optimization problem, thus a minimizer of the training error is in general not unique.

Finally, we point out that deep mathematical results on approximation properties of NODEs are presented in~\cite{Tabuada} and~\cite{Teshima}, building on~\cite{Li}. For example, in~\cite{Tabuada} it is shown that parameters can be determined in~\eqref{eq:kNODE} with $F$ in the form~\eqref{eq:F} to match the given output states with arbitrary precision in the $L^\infty$-norm. Limitations on what functions that can actually be approximated is presented in~\cite{Dupont}, where extension to more general systems of NODEs is suggested; see also~\cite{AvelinNystrom}.

\section{Numerical results}
\label{sec:numresults}
We recall that in this work the NODE is considered in isolation, that is describing the full network, where the input is plugged into the NODE as the initial condition and the output is generated as the final value of the NODE. The parameters to be found (learned) are allowed to depend on time (depth-variable).  We shall generate results on training this time-dependent model using the proposed NCG method for NODEs, and compare with the corresponding results obtained by conventional discrete methods utilized in deep learning. As a proof-of-concept, tests are performed on synthetically generated 2D data described below. 

The NCG is implemented in MATLAB using standard solvers for the involved NODEs (for example \texttt{ODE45}). We will make precise the NODE solved and the cost functional which is minimized. We start by first describing the datasets.
\subsection*{Two moons dataset}
\addcontentsline{toc}{subsection}{Two moons dataset}
The two moons dataset is a simple synthetic dataset in the Euclidean plane consisting of two semicircles, and is included in the \texttt{scikit-learn} library for Python. The upper semicircle of radius $1$ with center at the origin corresponds to one of the moons, while the lower semicircle of radius $1$ with center at the point $(1, 0.5)$ corresponds to the other moon. The input data $\vx^0 \in \Rbb^2$ consists of $x$- and $y$-coordinates of a point in the plane and the output $\vy \in \Rbb^2$ is a one-hot encoded category of the point, i.e., $\vy = (1,0)^\tran$ for points of the first moon, whereas $\vy = (0, 1)^\tran$ for points of the second moon.

The training dataset consists of a 1000 points that are uniformly distributed along the two moons and then perturbed by Gaussian noise with standard deviation of 0.07 (the training points are thus clustered around the two moons).
\subsection*{Two circles dataset}
\addcontentsline{toc}{subsection}{Two circles dataset}
Similarly as for the two moons, the two circles dataset is a standard simple synthetic dataset in the Euclidean plane, and is also included in the \texttt{scikit-learn} Python library. The dataset consists of a circle of radius $1$ and a circle of radius $0.5$, both with center at the origin. The input data $\vx^0 \in \Rbb^2$ consists of $x$- and $y$-coordinates of a point in the plane and the output $\vy \in \Rbb^2$ is a one-hot encoded category of the point, i.e., $\vy = (1,0)^\tran$ for points of the outer circle, whereas $\vy = (0, 1)^\tran$ for points of the inner circle.

The training dataset consists of a 1000 points that are uniformly distributed along the two circles and then perturbed by Gaussian noise with standard deviation of 0.07 (the training points are thus clustered around the two circles).
\subsection*{Test datasets}
\addcontentsline{toc}{subsection}{Test datasets}
To measure how well the respective neural network has learned to classify the two moons/circles, we ran inference on two distinct test datasets and computed the accuracy of the classification, i.e., the proportion of test points that were classified correctly.

One of the test sets had 100 points of ``clean'' data, meaning that the points were uniformly distributed on the two moons/circles. The other test set had 1000 points and consisted of ``noisy'' data, meaning that the points were uniformly distributed on the two moons/circles and then perturbed by a Gaussian noise of standard deviation 0.06. The noisy data was thus clustered around the two moons/circles.
\subsection*{Neural network structure and the cost functional}
\addcontentsline{toc}{subsection}{Neural network structure and the cost functional}
We use a neural ODE model on the interval $[0, T]$ based on fully connected linear layers with $\tanh$ as an activation function, namely,
\begin{equation}
  \label{eq:specNODE}
  \left\{
    \begin{aligned}
       \vx_k'(t) & = \tanh( W(t)\vx_k(t) + b(t)), \quad t \in (0, T), \\
       \vx_k(0) & = \vx_k^0
    \end{aligned}
  \right.
\end{equation}
with $k=1,2,\ldots, K$, where $W: [0,T] \to \Rbb^{2\times2}$ and $b: [0, T] \to \Rbb^2$ are the parameters to be found (trained). Note that the activation function $\tanh$ is applied component-wise. Theorem~\ref{thm:existence_uniqueness} applies and guarantees the well-posedness of~\eqref{eq:specNODE}. 
In our experiments, the ``time interval'' has length $T=5$ and data batches comprise $100$ points during the training and hence $K=100$. We point out that the parameters $W$ and $b$ can be vectorized and put into a generic parameter $\theta$, i.e.,
\[
  \theta(t) = \bigl(
  W_{1,1}(t), \ W_{2,1}(t), \ W_{1,2}(t), \ W_{2,2}(t), \ b_1(t), \ b_2(t)
  \bigr)^\tran \in \Rbb^6, \quad t\in [0, T],
\]
thus the framework of Section~\ref{sec:minimization} can be applied. 

For simplicity, the cost functional~\eqref{eq:E} is written in terms of the parameters $W$ and $b$, that is $E(W,b)$. Terms included in this cost functional are chosen as to match the output of the network $\vx_k(T)$ with the desired output $\vy_k$ possibly penalizing the magnitude of the functions $W(t)$ and $b(t)$ and their derivatives, thus
\begin{equation}
  \label{eq:specE}
	\begin{aligned}
		E(W,b) & = \frac{1}{K} \sum_{k=1}^K \biggl( \frac{\mu_1}{2} \| \vx_k(T) - \vy_k \|_{\ell^2}^2 + \mu_2 H\bigl(\vy_k, \smax (\vx_k(T))\bigr) + \frac{\mu_3}{2} \| \vx_k(T)\|_{\ell^2}^2\biggr) \\
		& \qquad + \int_0^T \biggl( \frac{\mu_4}{2} \bigl(\|W(t)\|^2_F + \|b(t)\|^2_{\ell^2}\bigr) + \frac{\mu_5}{2} \bigl(\|W'(t)\|^2_F + \|b'(t)\|^2_{\ell^2}\bigr) \biggr)\,dt,
	\end{aligned}
\end{equation}
with $\mu_1, \ldots, \mu_5 \ge 0$, where $H(\vek{p}, \vek{q}) = - \sum_i p_i \log(q_i)$ is the cross-entropy between the probability vectors $\vek{p}, \vek{q} \in \Rbb^N$ and $\| \cdot \|_F$ denotes the Frobenius norm of a matrix. Softmax is a standard function turning input into a probability vector required to apply the cross-entropy (explicit definition is in the next paragraph). Note that the first row in the right-hand side of~\eqref{eq:specE} corresponds to the loss function $L(\vx_k(T), \vy_k)$ and the second row to $\Phi(\theta(t), \theta'(t))$ in the general formula \eqref{eq:E} (the term involving $\ell$ is put to zero and is not included).

To match the output via the squared mean distance, set $\mu_1>0$ and $\mu_2=\mu_3=0$. If the matching is done via cross-entropy, then set $\mu_2>0$ and $\mu_1=0$. When using cross-entropy, it is desirable to control the magnitude of the output by $\mu_3>0$ since softmax has the unfortunate property to be invariant under shift along the space diagonal $\vek{d} = (1,1,\ldots,1)^\tran$, i.e.,
\[
  \smax(\vz) := \frac{1}{\sum_i \exp(z_i)} \begin{pmatrix} \exp(z_1)\\ \exp(z_2)\\ \vdots \\ \exp(z_N)\end{pmatrix} = \frac{1}{\sum_i \exp(z_i + c)} \begin{pmatrix} \exp(z_1 + c)\\ \exp(z_2+c)\\ \vdots \\ \exp(z_N + c)\end{pmatrix} = \smax(\vz + c \vek{d})
\]
for every $\vz \in \Rbb^N$ and $c\in \Rbb$.

Considering the NODE~\eqref{eq:specNODE} with the cost functional~\eqref{eq:specE}, we can now give specific formulae for the adjoint problem~\eqref{eq:AP}, the sensitivity problem \eqref{eq:SP}, the direction of steepest descent \eqref{eq:E'-L2} as well as the learning rate $\beta$.

Recall that the activation function $\tanh$ in \eqref{eq:specNODE} is applied component-wise. Therefore, the chain rule yields that the adjoint problem~\eqref{eq:AP} can be formulated as follows:
\begin{equation}
\label{eq:specAP}
	\left\{
	\begin{aligned}
	\lambda_k'(t) & = - W(t)^\tran \Bigl(\sech^2\bigl(W(t) \vx_k(t) + b(t)\bigr) \circ \lambda_k(t)\Bigr), \quad t\in (0,T),\\
	\lambda_k(T) & = \frac{1}{K} \Bigl( \mu_1 \bigl(\vx_k(T)-\vy_k\bigr) + \mu_2 \bigl(\smax(\vx_k(T))-\vy_k\bigr) + \mu_3\, \vx_k(T)\Bigr),
	\end{aligned}
	\right.
\end{equation}
where $\sech^2$ is applied component-wise and $\circ$ denotes the Hadamard (i.e., component-wise) product. Note also that we have made use of the fact that $\vy_k$ is a probability vector when simplifying the derivative of the cross-entropy term in the final value condition.

The terms that penalize $W$ and $b$ and their derivatives in \eqref{eq:specE} may be expressed using the Lebesgue $L^2$ and Sobolev $W^{1,2}$ norms, namely,
\begin{align*}
	& \int_0^T \biggl( \frac{\mu_4}{2} \bigl(\|W(t)\|^2_F + \|b(t)\|^2_{\ell^2}\bigr) + \frac{\mu_5}{2} \bigl(\|W'(t)\|^2_F + \|b'(t)\|^2_{\ell^2}\bigr) \biggr)\,dt \\
	& \quad = \frac{\mu_4 - \mu_5}{2} \Bigl(\| W \|_{L^2(I)}^2 + \| b \|_{L^2(I)}^2\Bigr) + \frac{\mu_5}{2} \Bigl(\| W \|_{W^{1,2}(I)}^2 + \| b \|_{W^{1,2}(I)}^2\Bigr).
\end{align*}
The Fréchet derivative of $E(W,b)$ with respect to the $L^2$ inner product has been stated in Corollary~\ref{cor:L2-derivative} corresponding to the case when $\mu_5 = 0$. Then, the direction of steepest ascent for $E(W,b)$ with respect to the $L^2$-norm is given by
\begin{equation}
\label{eq:L2_sd}
	\left\{
	\begin{aligned}
	dW(t) & = \sum_{k=1}^K \lambda_k(t) \circ \bigl( \sech^2(W(t) \vx_k(t) + b(t)) \vx_k(t)^\tran\bigr) + (\mu_4 - \mu_5) W(t),\\
	db(t) & = \sum_{k=1}^K \lambda_k(t) \circ \bigl( \sech^2(W(t) \vx_k(t) + b(t))\bigr)+ (\mu_4 - \mu_5) b(t).
	\end{aligned}
	\right.
\end{equation}
If $\mu_5 > 0$, then the cost functional $E(W,b)$ is not differentiable in $L^2(I)$. On the other hand, it is differentiable in $W^{1,2}(I)$. By Corollary~\ref{cor:SobolevPairing}, the direction of steepest ascent for $E(W,b)$ with respect to the $W^{1,2}$-norm is then given by
\begin{equation}
\label{eq:W12_sd}
 \delta W(t) = \Scal[dW](t) + \mu_5 W(t) \quad \text{and} \quad
 \delta b(t) = \Scal[db](t) + \mu_5 b(t),
\end{equation}
where $\Scal[\,\cdot\,]$ is the transformation described in Theorem~\ref{thm:L2-W12_convert} applied component-wise to the functions $dW$ and $db$ from \eqref{eq:L2_sd}.

Having determined a direction along which the weights and biases are to be updated, we can use the sensitivity problem to measure the change of the output of the neural network. If $W(t)$ and $b(t)$ are perturbed by $V(t)$ and $a(t)$, respectively, then \eqref{eq:SP} is expressed by
\begin{equation}
\label{eq:specSP}
	\left\{
	\begin{aligned}
		 \vxi_k'(t) & = \sech^2 \bigl( W(t)\vx_k(t)+b(t) \bigr) \circ \bigl(W(t) \vxi_k(t) + V(t) \vx_k(t) + a(t)\bigr), \quad t\in(0,T),\\
		 \vxi_k(0) & = 0.
	\end{aligned}
	\right.
\end{equation}
The auxiliary functional $\widetilde{E}(\beta)$ that approximates $E(W+\beta\mkern 1mu V, b+\beta\mkern 1mu a)$
has derivative
\begin{equation}
\label{eq:specEtilde'}
	\begin{aligned}
	 \widetilde{E}'(\beta) & = \frac{1}{K} \sum_{k=1}^K
	 \biggl\langle \mu_1 \bigl( \vx_k(T) + \beta \vxi_k(T) - \vy_k\bigr) \\
	 & \qquad\qquad\quad + \mu_2 \bigl( \smax(\vx_k(T) + \beta \vxi_k(T)) - \vy_k \bigr) + \mu_3 \vx_k(T), \vxi_k(T)\Bigr\rangle_{\ell^2} \\
	 & \quad + \int_0^T \biggl(\mu_4 \Bigl( \bigl\langle W(t) + \beta V(t), V(t)\bigr\rangle_F + \bigl\langle b(t) + \beta a(t), a(t)\bigr\rangle_{\ell^2} \Bigr) \\
	 & \qquad\qquad\quad + \mu_5 \Bigl( \bigl\langle W'(t) + \beta V'(t), V'(t)\bigr\rangle_F + \bigl\langle b'(t) + \beta a'(t), a'(t)\bigr\rangle_{\ell^2} \Bigr)
	 \biggr) dt,
	\end{aligned}
\end{equation}
where $\langle\cdot, \cdot \rangle_F$ denotes the Frobenius inner product on $\Rbb^{N\times N}$.
The optimal learning rate is found by solving the equation $\widetilde{E}'(\beta) = 0$. Note that this equation is linear in $\beta$ in the case when $\mu_2 = 0$.
\subsection*{Training process}
\addcontentsline{toc}{subsection}{Training process}
At the very beginning of the training, the sought parameters $W_0(t)$ and $b_0(t)$ were initialized to an arbitrary $\Ccal^1$-function of a small Sobolev $W^{1,2}$-norm and a training dataset of 1000 points was generated/loaded.
In each epoch (passes of the entire training dataset), the training data were divided into 10 batches, each containing $50 + 50$ points of the two moons/circles. Each batch was used in 15 iterations of the gradient descent steps as described in Section~\ref{sec:minimization}:

In each iteration, the direct problem~\eqref{eq:specNODE} and then the adjoint problem~\eqref{eq:specAP} were solved. The $L^2$-steepest ascent direction was computed using \eqref{eq:L2_sd}. In the first iteration for a batch, the chosen descent direction was the steepest one, i.e., $V_0(t) = -dW_0(t)$ and $a_0(t)=-db_0(t)$. In the $(j+1)^\text{st}$ iteration, the conjugate gradient coefficient was calculated by
\[
  \gamma_j = \frac{\|dW_j\|_{L^2(I)}^2 + \|db_j\|_{L^2(I)}^2}{\|dW_{j-1}\|_{L^2(I)}^2 + \|db_{j-1}\|_{L^2(I)}^2}
\]
and the descent direction was given by
\[
  V_j(t) = -dW_j(t) + \gamma_j V_{j-1}(t) \quad \text{and}\quad
  a_j(t) = -db_j(t) + \gamma_j a_{j-1}(t).
\]
The sensitivity problem~\eqref{eq:specSP} was solved and the learning rate $\beta_j$ was determined as the zero of $\widetilde{E}'(\beta)$ in \eqref{eq:specEtilde'}. Then, the parameters were updated so that $W_{j+1}(t) = W_j(t) + \beta_j V_j(t)$ and $b_{j+1}(t) = b_j(t) + \beta_j a_j(t)$. Thereafter, the next iteration was started.

In the case when the Sobolev-gradient direction was to be used, then \eqref{eq:L2_sd} in the algorithm above was immediately followed by \eqref{eq:W12_sd}, all instances of $dW$ and $db$ were replaced by $\delta W$ and $\delta b$, respectively, and the conjugate gradient coefficient was calculated using the $W^{1,2}$-norm instead of the $L^2$-norm.

When 15 iterations were finished for a batch, then another batch of $50+50$ training data points within the epoch was taken and used in 15 iterations. When all 10 batches of an epoch were used, then a new epoch with a new partition of the training data into batches was started. The weights and biases $W_{15}$ and $b_{15}$ obtained after one cycle of iterations are used as starting values $W_0$ and $b_0$ for the next cycle of iterations. 
\subsection*{Inference}
\addcontentsline{toc}{subsection}{Inference}
When a given set of data points $\{ \vx_k^0 \in \Rbb^N : k=1,\ldots, K\}$ is to be classified using the neural network (with the learned parameters), it is fed into the network via the initial condition in the NODE~\eqref{eq:specNODE}. The final value $\vx_k(T)$ of the solution to the differential equation is then used to determine the class of the respective point. In our experiments, we used one-hot encoding for the desired output, where the first class was encoded as the vector $(1, 0)^\tran$, whereas the second class was encoded as $(0,1)^\tran$. The $k^{\text{th}}$ point belongs to the first class if the distance from $\vx_k(T)$ to $(1,0)^\tran$ is shorter than the distance to $(0,1)^\tran$. Otherwise, the point belongs to the second class.
\subsection*{Augmentation for the two circles dataset}
\addcontentsline{toc}{subsection}{Augmentation for the two circles dataset}
Consider the transformation that maps a point $\vx^0 \in \Rbb^2$ to the final value $\vx(T) \in \Rbb^2$ of the solution to the NODE~\eqref{eq:specNODE} having $\vx^0$ as its initial condition. As a consequence of Theorem~\ref{thm:existence_uniqueness}, this mapping is an orientation-preserving homeomorphism of the plane. In particular, when this transformation is applied to the circles, then the image of the inner circle stays inside the image of the outer circle. Therefore, the image of certain points of the outer circle lies necessarily closer to $(0,1)^\tran$ and these points are thereby incorrectly classified. This is an inevitable property of the continuous neural network model due to the topology of the dataset and the plane. A possible solution to the topological constraints is to embed the plane in the 3-dimensional space by padding both the input and the desired output by (e.g.) zeroes. When the dimension of the data increases, the dimension of the trainable parameters has to increase accordingly. Hence, $W: [0,T] \to \Rbb^{3\times 3}$ and $b: [0,T] \to \Rbb^3$ are the trainable parameters for the augmented neural differential equation (this type of augmentation method is also considered in~\cite{Dupont}).
\subsection*{Comparison of numerical results of the NCG for various cost functionals}
\addcontentsline{toc}{subsection}{Comparison of numerical results of the NCG for various cost functionals}
For each of the synthetic datasets (two moons, two circles without augmentation, and two circles with augmentation), we trained a NODE-based neural network in 5 epochs using several combinations of terms in the cost functional. We measured the best achieved accuracy for noisy test data as well as for clean test data (the word clean here means that no noise has been applied, see the description of test sets above), and we recorded at which epoch/batch this accuracy was obtained. This means that training for any additional epochs/batches does not lead to improvement in the achieved accuracy (some improvement might be obtained after the $5^\text{th}$ epoch). The results over 10 independent test runs for each combination are collected in Tables~\ref{tbl:AccL2} and~\ref{tbl:AccW12}. Some tests had clustered values with a tail to left, giving a large standard deviation. No accuracies can of course be higher than 100. The decimal part in the epoch count shows the number of batches before completing an epoch, e.g., epoch count 1.6 corresponds to having trained for 1 full epoch and 6 batches (out of 10) of the next epoch.

For the two moons dataset, square mean distance was used to measure the difference between the network output $\vx_k(T)$ and the one-hot encoded output $\vy_k$ and hence $\mu_1 = 1$ while $\mu_2=\mu_3=0$. For the two circles dataset (regardless of whether augmentation is applied or not), cross-entropy was used and hence $\mu_1=0$ while $\mu_2 = 1$ and $\mu_3 = 0.1$.
We made this choice in order to illustrate results for matching the output both by a norm, and by cross-entropy. As explained above, due to the shift invariance along the space diagonal of the cross-entropy, it is advisable to have $\mu_3>0$. 

We tested three types of penalization of the parameters $W$ and $b$ in our experiments, namely, no penalization at all ($\mu_4 = \mu_5 = 0$), $L^2$-penalization ($\mu_4 = 10^{-5}$ and $\mu_5=0$), and $W^{1,2}$-penalization ($\mu_4 = \mu_5 = 10^{-5}$).
The particular choice of the value $10^{-5}$ was decided empirically so that the penalization of parameters does not overpower the terms that measure matching of the output during the iterative optimization process.
\begin{table}[h!]
\caption{\noindent Averages and standard deviations of highest achieved accuracies within 5 epochs over 10 test runs when using $L^2$-steepest descent direction during the optimization.}
\label{tbl:AccL2}
\centering \renewcommand{\arraystretch}{1.15}
\begin{tabular}{c|c|c|c|c|c|}
\multicolumn{2}{c|}{ } & \multicolumn{2}{c|}{No penalization by $\Phi$} & \multicolumn{2}{c|}{$L^2$-penalization}  \\ \hline
Dataset & Testdata & Best accuracy & Epoch & Best accuracy & Epoch \\ \hline\hline
\multirow{2}{*}{\shortstack[c]{Moons\\(2D)}}
 & Clean & $100 $ & $  0.6 \pm 0.3 $ & $ 100 $ & $  0.7 \pm 0.2$  \\ \cline{2-6}
 & Noisy & $100 $ & $  0.9 \pm 0.6 $ & $ 100 $ & $  1.1 \pm 0.9$  \\ \hline
\multirow{2}{*}{\shortstack[c]{Circles\\(2D)}}
 & Clean & $ 98.1 \pm 2.8 $ & $  2.0 \pm 0.7 $ & $  93.0 \pm 12 $ & $  1.4 \pm 1.2$  \\ \cline{2-6}
 & Noisy & $ 97.6 \pm 3.1 $ & $  3.4 \pm 1.3 $ & $  92.6 \pm 12 $ & $  1.6 \pm 1.3$  \\ \hline
\multirow{2}{*}{\shortstack[c]{Circles\\(3D)}}
 & Clean & $100 $ & $  0.4 \pm 0.05 $ & $ 100 $ & $  0.6 \pm 0.3$ \\ \cline{2-6}
 & Noisy & $100 $ & $  0.5 \pm 0.05 $ & $  99.99 \pm 0.03 $ & $  1.1 \pm 1.4$ \\ \hline
\end{tabular}
\end{table}
\begin{table}[h!]
\noindent{\textcolor{gray}{\rule{\textwidth}{0.3mm}}}
\caption{\noindent Averages and standard deviations of highest achieved accuracies within 5 epochs over 10 test runs when using $W^{1,2}$-steepest descent direction. Remark: 100\% accuracy for clean test data for augmented circles was achieved in all test runs within 13 epochs.}
\label{tbl:AccW12}
\centering \renewcommand{\arraystretch}{1.15}
\begin{tabular}{c|c|c|c|c|c|c|c|}
\multicolumn{2}{c|}{ } & \multicolumn{2}{c|}{No penalization by $\Phi$} & \multicolumn{2}{c|}{$L^2$-penalization} & \multicolumn{2}{c|}{$W^{1,2}$-penalization}  \\ \hline
\shortstack[c]{Data\\set} \rule{0pt}{4.5ex} & \shortstack[c]{Test\\data} & \shortstack[c]{Best\\[2pt]accuracy} & Epoch & \shortstack[c]{Best\\[2pt]accuracy} & Epoch & \shortstack[c]{Best\\[2pt]accuracy} & Epoch \\ \hline\hline
\multirow{2}{*}{\shortstack[c]{Moons\\(2D)}}
 & Clean & $100 $ & $  2.7 \pm 0.4 $ & $ 100 $ & $  2.7 \pm 0.5 $ & $ 100 $ & $  2.6 \pm 0.7$  \\ \cline{2-8}
 & Noisy & $100 $ & $  4.0 \pm 0.7 $ & $  99.99 \pm 0.03 $ & $  3.9 \pm 0.7 $ & $  99.9 \pm 0.1 $ & $  3.8 \pm 0.7$  \\ \hline
\multirow{2}{*}{\shortstack[c]{Circles\\(2D)}}
 & Clean & $ 97.2 \pm 2.3 $ & $  2.2 \pm 1.6 $ & $  96.3 \pm 3.2 $ & $  2.8 \pm 1.4 $ & $  96.1 \pm 3.1 $ & $  2.0 \pm 0.8$  \\ \cline{2-8}
 & Noisy & $ 96.3 \pm 2.6 $ & $  2.7 \pm 1.7 $ & $  95.2 \pm 3.4 $ & $  3.2 \pm 1.2 $ & $  95.3 \pm 3.3 $ & $  3.3 \pm 1.4$ \\ \hline
\multirow{2}{*}{\shortstack[c]{Circles\\(3D)}}
 & Clean & $ 99.5 \pm 1.6 $ & $  1.6 \pm 1.2 $ & $  99.3 \pm 2.2 $ & $  1.6 \pm 1.6 $ & $  99.4 \pm 1.6 $ & $  1.8 \pm 1.4$ \\ \cline{2-8}
 & Noisy & $ 99.1 \pm 2.2 $ & $  2.2 \pm 1.6 $ & $  98.9 \pm 2.9 $ & $  2.2 \pm 1.7 $ & $  99.0 \pm 2.6 $ & $  2.4 \pm 1.8$ \\ \hline
\end{tabular}
\end{table}

\clearpage
In Table~\ref{tbl:AccL2}, the steepest descent direction was determined with respect to the $L^2$-pairing, i.e., \eqref{eq:L2_sd} was applied.
In Table~\ref{tbl:AccW12}, the steepest descent direction was given with respect to the $W^{1,2}$-pairing, i.e., \eqref{eq:W12_sd} was used.
Recall that $E(W,b)$ is not differentiable in $L^2$ if $\mu_5>0$, which is why the column corresponding to $W^{1,2}$-penalization is not included in Table~\ref{tbl:AccL2}.

As can be seen from Tables~\ref{tbl:AccL2} and~\ref{tbl:AccW12}, the various penalization terms do not improve or significantly change the results, compared to when no penalization is used. This phenomenon has been observed earlier for conjugate gradient methods, see for example~\cite{Hao,cao3}. When penalization is present, then the iterations can in principle continue without any risk of the parameters blowing up, whilst with no penalization the iterations have to be terminated using an appropriate stopping rule.

It is also apparent from Table~\ref{tbl:AccL2} and Table~\ref{tbl:AccW12} that the accuracy is higher and reached in fewer epochs/batches when the $L^2$-steepest descent direction is used.

To gain further insight into these results, we computed the corresponding Lebesgue $L^2$ and Sobolev $W^{1,2}$-norms of the trainable parameters at the end of the $5^\text{th}$ epoch. Averages and standard deviations of the norms over 10 test runs are recorded in Table~\ref{tbl:L2norms} ($L^2$-norms) and Table~\ref{tbl:W12norms} ($W^{1,2}$-norms).
\begin{table}[h]
\caption{Averages and standard deviations of $(\|W\|_{L^2(I)}^2 + \|b\|_{L^2(I)}^2)^{1/2}$ after 5 epochs over 10 test runs}
\label{tbl:L2norms}
\centering \renewcommand{\arraystretch}{1.15}
\begin{tabular}{c|c|c|c|c|c|}
  & \multicolumn{2}{c|}{No penalization by $\Phi$} & \multicolumn{2}{c|}{$L^2$-penalization} & $W^{1,2}$-pen.  \\ \hline
Dataset & $L^2$-descent & $W^{1,2}$-descent & $L^2$-descent & $W^{1,2}$-descent & $W^{1,2}$-descent \\ \hline\hline
Moons (2D) & $16.7 \pm 3.6 $ & $ 8.15 \pm 0.66 $ & $ 16.3 \pm 4.4 $ & $ 8.06 \pm 0.37 $ & $ 8.13 \pm 0.55$ 
 \\ \hline
Circles (2D) & $46.0 \pm 56.1 $ & $ 15.1 \pm 7.5 $ & $ 355 \pm 236 $ & $ 13.0 \pm 4.6 $ & $ 13.0 \pm 4.6$ 
 \\ \hline
Circles (3D)  & $7.47 \pm 1.53 $ & $ 23.4 \pm 34.2 $ & $ 13.6 \pm 3.4 $ & $ 21.5 \pm 27.9 $ & $ 20.9 \pm 27.5$  \\ \hline
\end{tabular}
\end{table}
\begin{table}[h]
\vskip-2.2ex
\noindent{\textcolor{gray}{\rule{\textwidth}{0.3mm}}}
\vskip1.3ex
\caption{Averages and standard deviations of $(\|W\|_{W^{1,2}(I)}^2 + \|b\|_{W^{1,2}(I)}^2)^{1/2}$ after 5 epochs over 10 test runs}
\label{tbl:W12norms}
\centering \renewcommand{\arraystretch}{1.15}
\begin{tabular}{c|c|c|c|c|c|}
  & \multicolumn{2}{c|}{No penalization by $\Phi$} & \multicolumn{2}{c|}{$L^2$-penalization} & $W^{1,2}$-pen.  \\ \hline
Dataset & $L^2$-descent & $W^{1,2}$-descent & $L^2$-descent & $W^{1,2}$-descent & $W^{1,2}$-descent \\ \hline\hline
Moons (2D) & $247 \pm 97 $ & $ 12.6 \pm 1.3 $ & $ 225 \pm 126 $ & $ 12.5 \pm 1.2 $ & $ 12.5 \pm 1.3$ 
  \\ \hline
Circles (2D) & $2111 \pm 3938 $ & $ 22.2 \pm 10.9 $ & $ 25786 \pm 18003 $ & $ 18.4 \pm 7.3 $ & $ 18.3 \pm 7.1$  \\ \hline
Circles (3D) & $80.6 \pm 35.4 $ & $ 28.2 \pm 38.1 $ & $ 180 \pm 113 $ & $ 26.0 \pm 30.4 $ & $ 25.4 \pm 30.4$ 
 \\ \hline
\end{tabular}
\end{table}

From Tables~\ref{tbl:L2norms} and~\ref{tbl:W12norms}, it can be seen that a penalizing term can in some instances  reduce the norms, and thus decrease the magnitude of the trained parameters and/or their derivative. However, the choice of the gradient descent direction has a much more significant effect. It is evident, in particular in Table~\ref{tbl:W12norms}, that using the $W^{1,2}$-steepest descent direction~\eqref{eq:W12_sd} (and possibly also a $W^{1,2}$-penalty term) greatly reduces the norm of the derivatives $W'(t)$ and $b'(t)$, which means that the oscillations in the learned parameters are eased out. Having smoother parameters improves speed and accuracy of the ODE-solvers in the NCG. Thus, although higher accuracy is reached in fewer epochs with the $L^2$-direction, the oscillations in the parameters slow down each iteration as the ODE-solver needs to repeatedly refine the discretization, making the results with $W^{1,2}$-direction comparable in usage of CPU time.

Unsurprisingly, Tables~\ref{tbl:AccL2} and~\ref{tbl:AccW12} confirm that least accuracy is achieved for the two circles dataset with no augmentation and the learned parameters have usually the highest norms for that dataset as seen in Tables~\ref{tbl:L2norms} and~\ref{tbl:W12norms}. Topology of the Euclidean plane and two concentric circles together with the continuous neural network model prevent successful classification. The inner circle has to stay inside the outer circle when undergoing transformation via the entire NODE, thus the circles cannot be linearly separated. As the class is decided by which of the points $(1,0)$ or $(0,1)$ lies nearest to the output of the network, a linear separation is needed.
\subsection*{Visualization of the results of NCG and comparison with SGD}
\addcontentsline{toc}{subsection}{Visualization of the results of NCG and comparison with SGD}
Each model trained by the proposed NCG was replicated in the Tensorflow machine learning platform using the Keras interface (Python based), in order to enable comparisons with existing standard stochastic gradient descent (SGD) based optimization techniques used in deep learning applications. The chosen reference models consists of 250 layers, which corresponds to the discretization level used in the NODE results. The activation function was scaled by the factor 1/50, in order to have an exact copy of a NODE solved with Euler discretization in the time interval $[0,5]$. All other hyper-parameters were also set as in the NODE optimization, with the same layer definitions, activation functions, loss functions, and batch size used in the training. For training, the SGD optimizer RMSProp~\cite{Hinton2012} was used, with a base learning rate of 0.1. 

For the two moon dataset and the augmented two circles dataset, training was performed over 15 epochs, reaching 100\% training accuracy. For the two circles dataset in 2 dimensions, training was stopped after 30 epochs with a training accuracy of around 98\%.

For visual comparison of the results obtained with the various optimization methods, Figures \ref{fig:moons}--\ref{fig:circles3D} have been generated (i.e., one figure per dataset: two moons, two circles without augmentation, and two circles with augmentation). Each of the three figures is formed as a 3-by-3 matrix of subfigures, where the first column shows results of NCG when using $L^2$-steepest descent direction \eqref{eq:L2_sd}, the middle column shows results of NCG with $W^{1,2}$-steepest descent direction \eqref{eq:W12_sd}, and the last column depicts results for the discrete neural network model trained by SGD. The used penalization $\Phi$ for NCG is listed in the footer of respective column. In each 3-by-3 matrix of subfigures the rows are as follows.
 
The first row contains graphs of all the trainable parameters $\theta(t)$ (consisting of weight matrices $W(t)$ and bias vectors $b(t)$) over the ``time'' interval $I=[0,5]$ for NCG and the corresponding $250$ layers used in SGD.
We point out that the SGD is an optimization method for a discrete model (not based on NODEs), thus the weights and biases are discrete. To further accentuate this, only every third value of the parameters $W$ and $b$ have been included in the graph for SGD so that they would not be misread as continuous curves by accident. 

The middle row shows decision boundaries obtained by respective neural network model. In order to produce the image of the decision boundaries, a grid of data points was generated in the shown rectangular area, with the resolution of 400-by-400 points per unit square. Every grid point was plugged in as input (i.e., as initial condition $\vx_k(0)$) into the neural network whose output (i.e., the final value $\vx_k(5)$) was then used to determine the color shade of the grid point. The difference of $y$- and $x$-coordinates of the output value represent the difference in distance from the points $(1,0)$ and $(0,1)$. Recall that these two points correspond to the two distinct classes as per one-hot encoding. On top of the decision boundaries, $250+250$ points of the noisy test dataset have been plotted and colored according to the ground truth. Note that this amounts to half of the noisy test set.

The final row shows trajectories of $25 + 25$ points of the clean test dataset as their position $\vx_k(t)$ evolve throughout the neural network. In Figures~\ref{fig:moons} and~\ref{fig:circles2D}, the time-component is preserved in the image and the value of $t$ is present as one of the axes. Observe that the trajectories do not cross, as distinct points are guaranteed to never collide in the NODE-model due to uniqueness of solutions of ODEs established in Theorem~\ref{thm:existence_uniqueness}. 
\begin{figure}[hb]
{\footnotesize Graphs of $W(t)$ and $b(t)$:}

\vskip1ex
\noindent
\begin{minipage}[t]{0.329\linewidth}
\includegraphics[width=\linewidth]{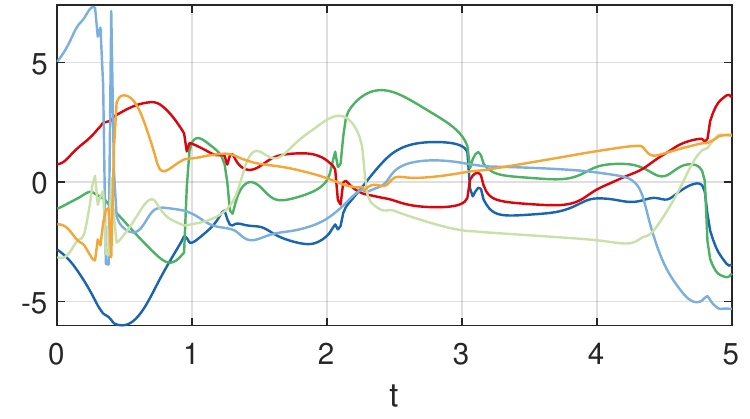}
\end{minipage}
\begin{minipage}[t]{0.329\linewidth}
\includegraphics[width=\linewidth]{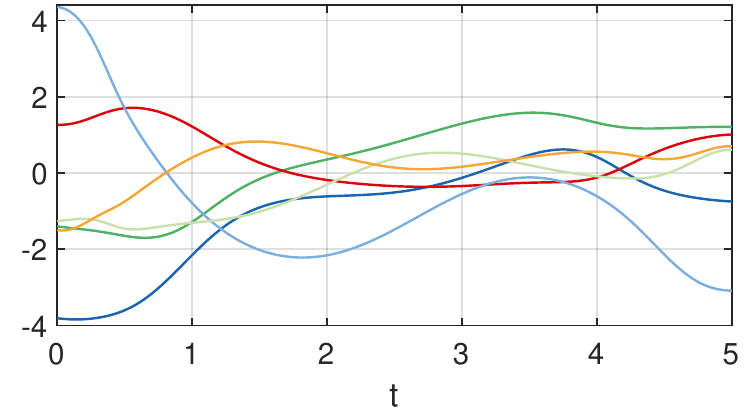}
\end{minipage}
\begin{minipage}[t]{0.329\linewidth}
\includegraphics[width=\linewidth]{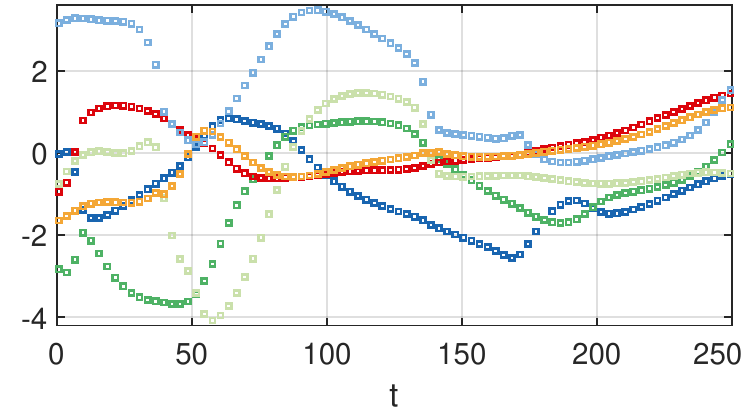}
\end{minipage}

\noindent
\vphantom{.}\hfill
\includegraphics[height=7.6mm]{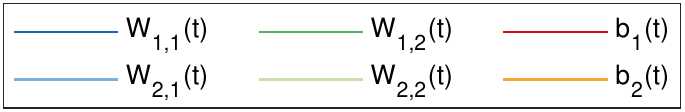}\hfill
\includegraphics[height=7.6mm]{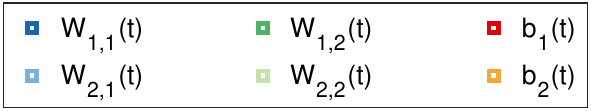}

\vskip3ex
{\footnotesize Decision boundaries:}

\vskip1ex
\noindent
\begin{minipage}[t]{0.329\linewidth}
\includegraphics[width=\linewidth]{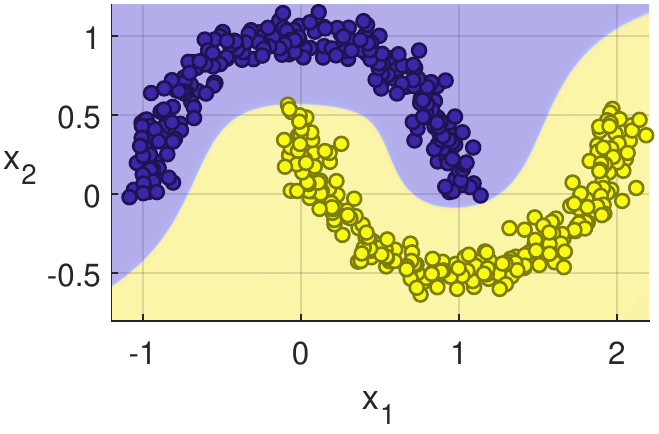}
\end{minipage}
\begin{minipage}[t]{0.329\linewidth}
\includegraphics[width=\linewidth]{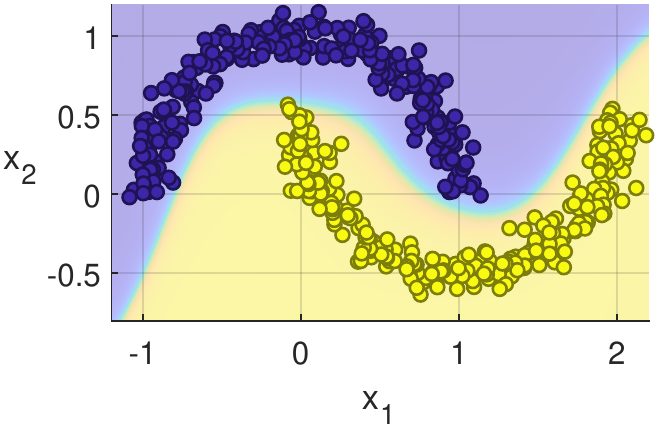}
\end{minipage}
\begin{minipage}[t]{0.329\linewidth}
\includegraphics[width=\linewidth]{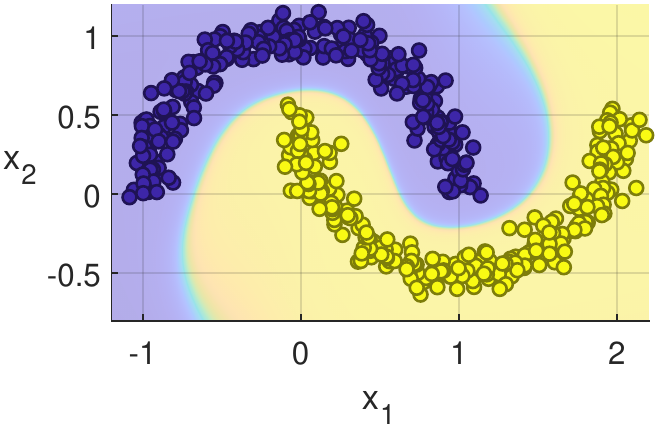}
\end{minipage}

\vskip3ex
{\footnotesize Trajectories:}

\vskip-1ex
\noindent
\begin{minipage}[t]{0.329\linewidth}
\includegraphics[width=2in]{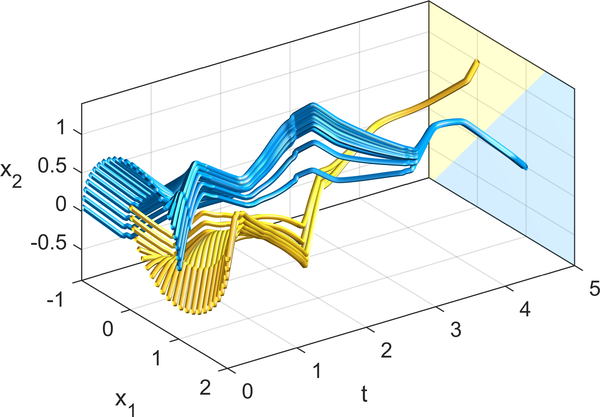}
\end{minipage}
\begin{minipage}[t]{0.329\linewidth}
\includegraphics[width=2in]{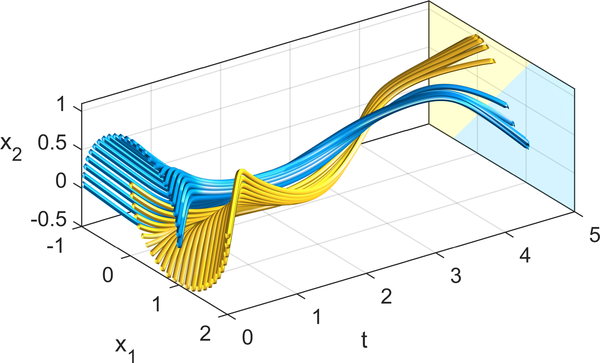}
\end{minipage}
\begin{minipage}[t]{0.329\linewidth}
\includegraphics[width=2in]{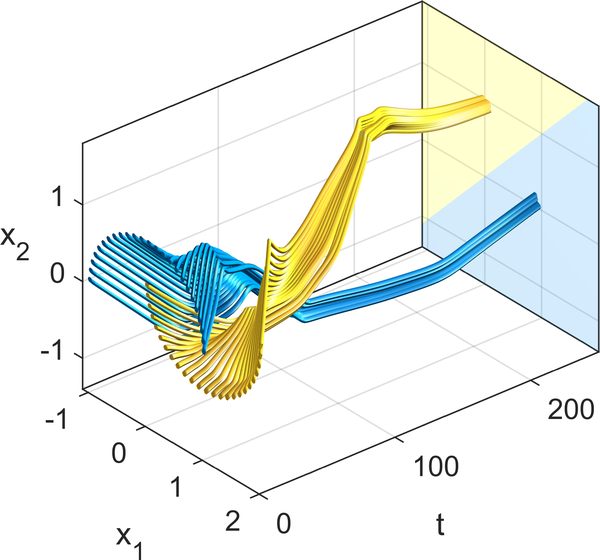}
\end{minipage}

\noindent
\begin{minipage}[t]{0.329\linewidth}
\begin{center}NCG with $L2$-descent

$L^2$-penalization\end{center}
\end{minipage}
\begin{minipage}[t]{0.329\linewidth}
\begin{center}NCG with $W^{1,2}$-descent

$W^{1,2}$-penalization\end{center}
\end{minipage}
\begin{minipage}[t]{0.329\linewidth}
\begin{center}SGD\end{center}
\end{minipage}

\caption{The two moons dataset. NCG depicted after batch 1 of epoch 4 in test run \#9.}
\label{fig:moons}
\end{figure}
\begin{figure}[t!]
{\footnotesize Graphs of $W(t)$ and $b(t)$:}

\vskip0.8ex
\noindent
\begin{minipage}[t]{0.329\linewidth}
\includegraphics[width=\linewidth]{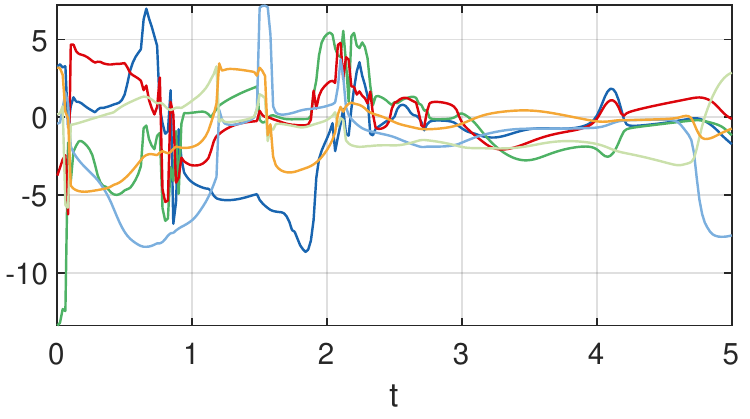}
\end{minipage}
\begin{minipage}[t]{0.329\linewidth}
\includegraphics[width=\linewidth]{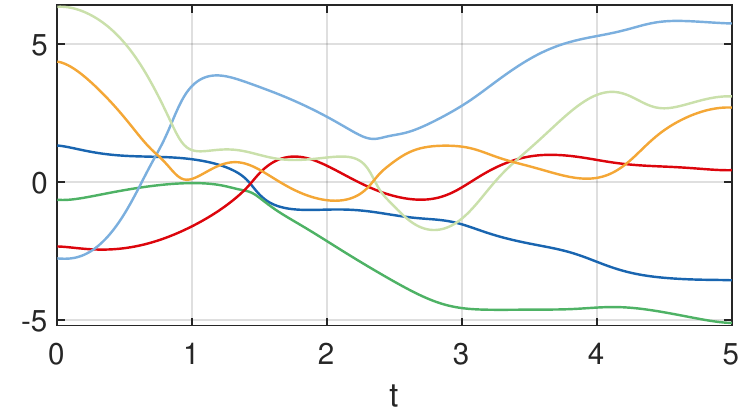}
\end{minipage}
\begin{minipage}[t]{0.329\linewidth}
\includegraphics[width=\linewidth]{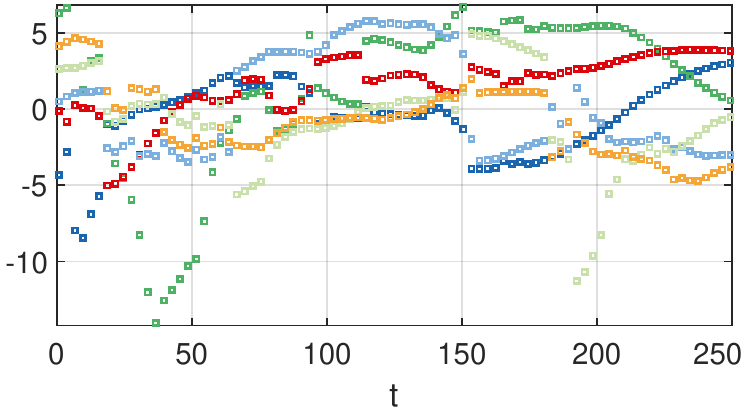}
\end{minipage}

\noindent
\vphantom{.}\hfill
\includegraphics[height=7.6mm]{figs_ncgLeb/2d_legend.pdf}\hfill
\includegraphics[height=7.6mm]{figs_sgd/2d_legend.pdf}

\vskip2ex
{\footnotesize Decision boundaries:}

\vskip0.8ex
\noindent
\begin{minipage}[t]{0.329\linewidth}
\includegraphics[width=\linewidth]{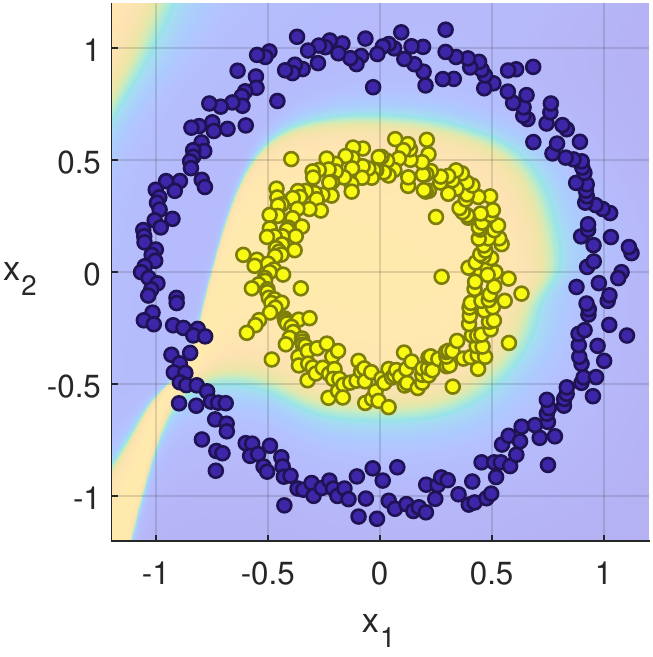}
\end{minipage}
\begin{minipage}[t]{0.329\linewidth}
\includegraphics[width=\linewidth]{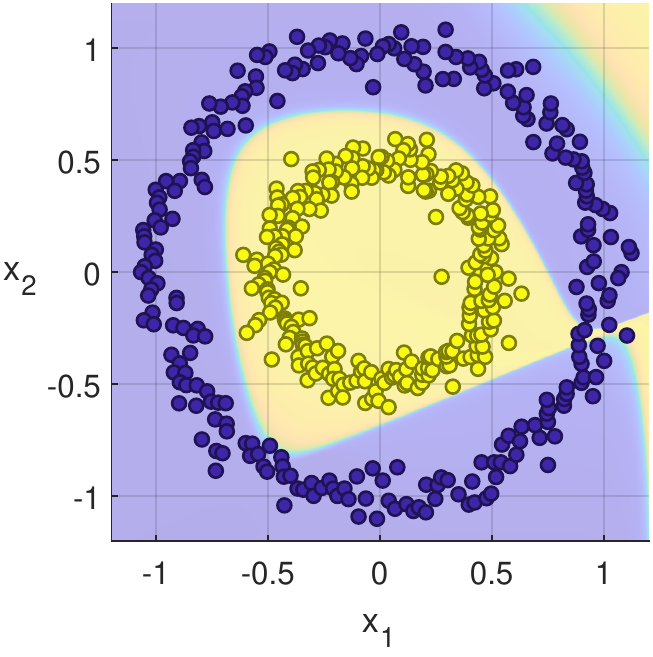}
\end{minipage}
\begin{minipage}[t]{0.329\linewidth}
\includegraphics[width=\linewidth]{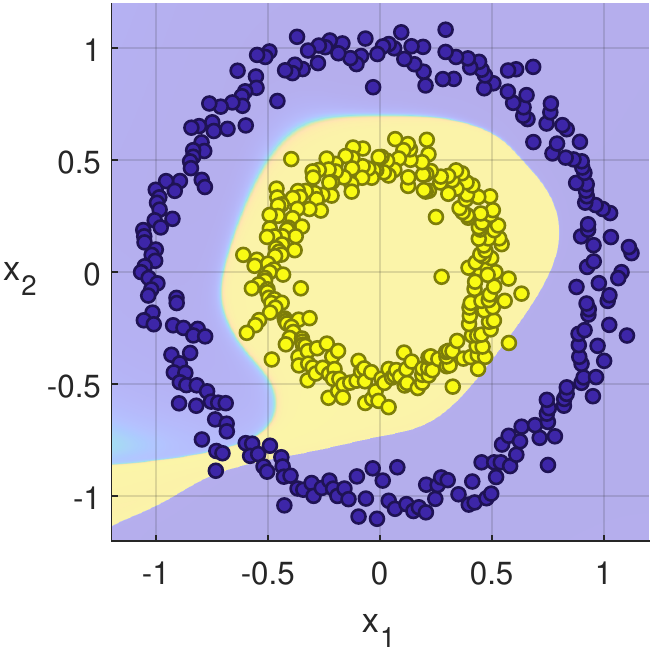}
\end{minipage}

\vskip1.5ex
{\footnotesize Trajectories:}

\noindent
\begin{minipage}[t]{0.329\linewidth}
\includegraphics[width=2in]{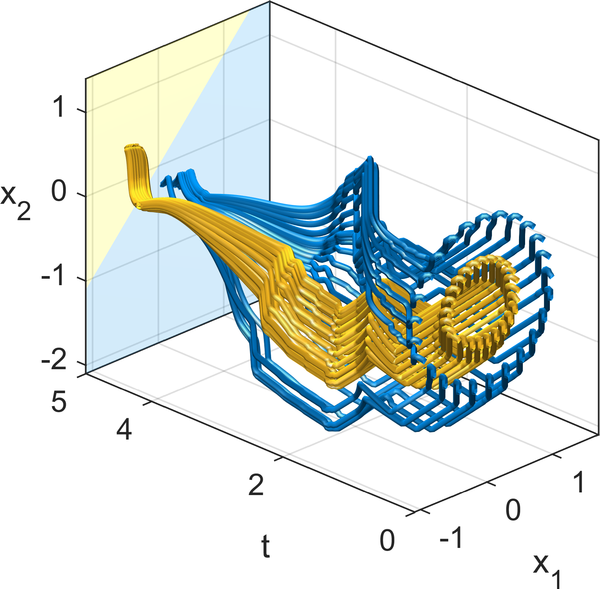}
\end{minipage}
\begin{minipage}[t]{0.329\linewidth}
\includegraphics[width=2in]{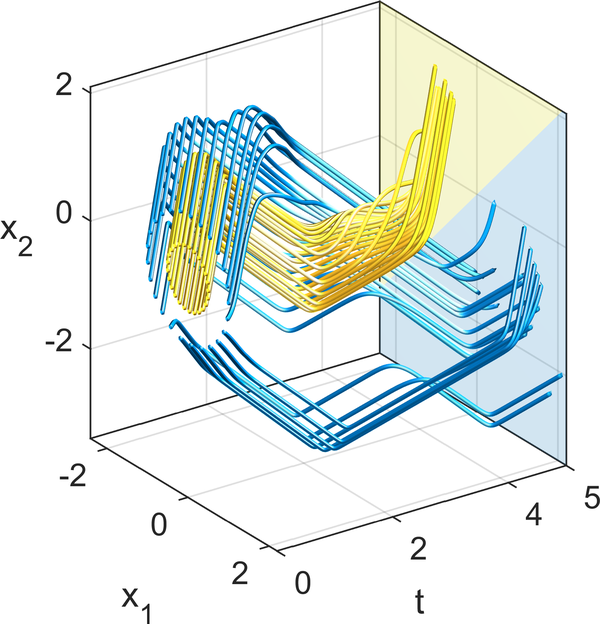}
\end{minipage}
\begin{minipage}[t]{0.329\linewidth}
\includegraphics[width=2in]{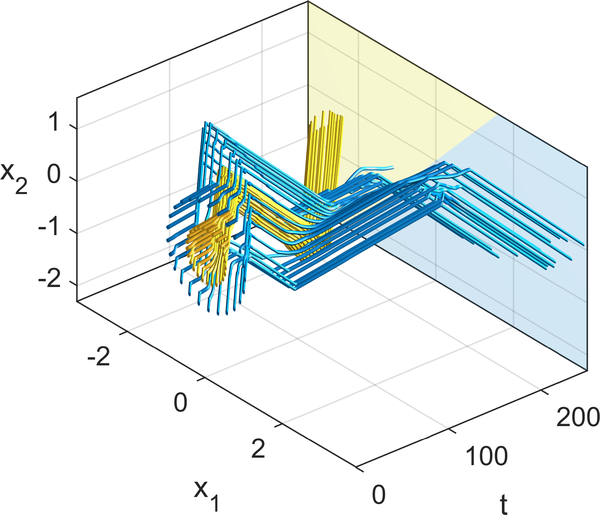}
\end{minipage}

\noindent
\begin{minipage}[t]{0.329\linewidth}
\begin{center}NCG with $L^2$-descent

$L^2$-penalization\end{center}
\end{minipage}
\begin{minipage}[t]{0.329\linewidth}
\begin{center}NCG with $W^{1,2}$-descent

No penalization by $\Phi$\end{center}
\end{minipage}
\begin{minipage}[t]{0.329\linewidth}
\begin{center}SGD\end{center}
\end{minipage}
\caption{The two circles dataset without augmentation. NCG shown after batch 2 of epoch 3 in test run \#2.}
\label{fig:circles2D}
\end{figure}
\begin{figure}[hbt!]
{\footnotesize Graphs of $W(t)$ and $b(t)$:}

\vskip1ex
\noindent
\begin{minipage}[t]{0.329\linewidth}
\includegraphics[width=\linewidth]{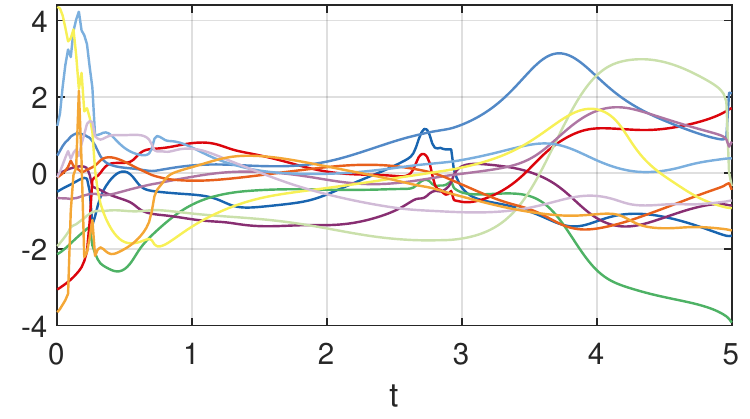}
\end{minipage}
\begin{minipage}[t]{0.329\linewidth}
\includegraphics[width=\linewidth]{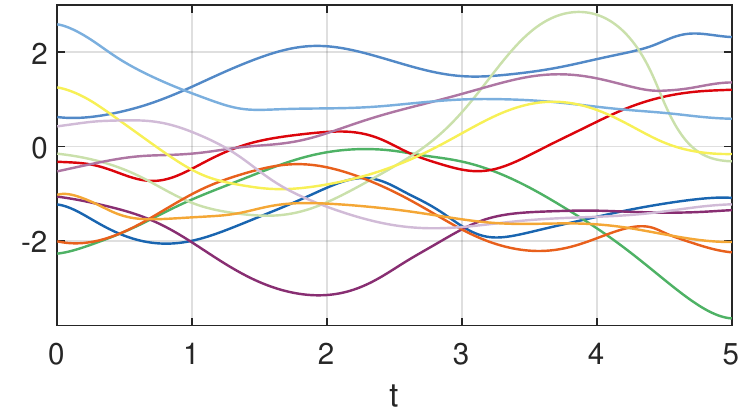}
\end{minipage}
\begin{minipage}[t]{0.329\linewidth}
\includegraphics[width=\linewidth]{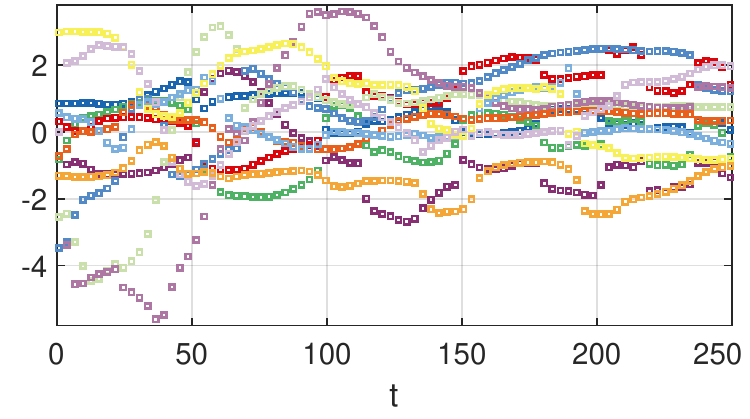}
\end{minipage}

\noindent
\vphantom{.}\hfill
\includegraphics[height=10.9mm]{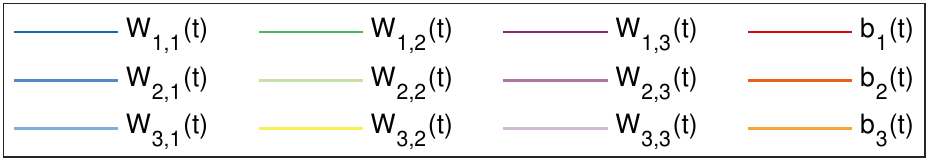}\hfill
\includegraphics[height=10.9mm]{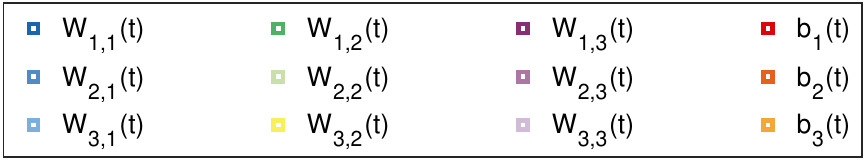}

\vskip3ex
{\footnotesize Decision boundaries:}

\vskip1ex
\noindent
\begin{minipage}[t]{0.329\linewidth}
\includegraphics[width=\linewidth]{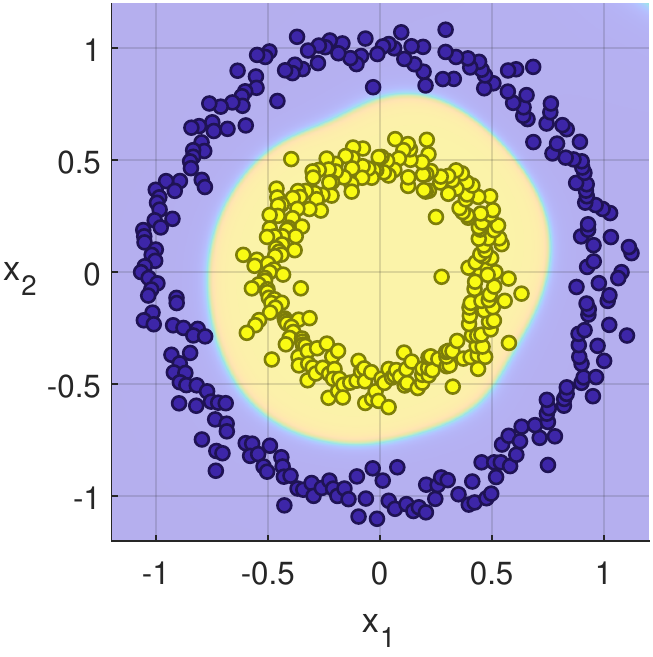}
\end{minipage}
\begin{minipage}[t]{0.329\linewidth}
\includegraphics[width=\linewidth]{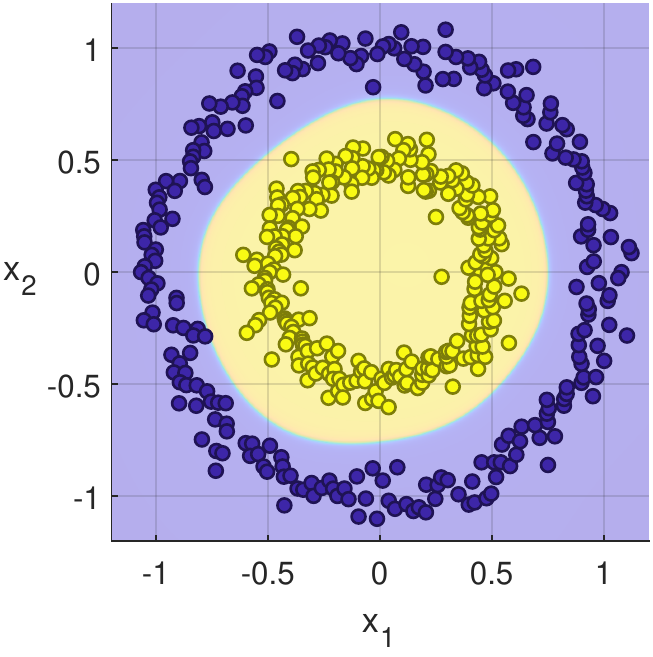}
\end{minipage}
\begin{minipage}[t]{0.329\linewidth}
\includegraphics[width=\linewidth]{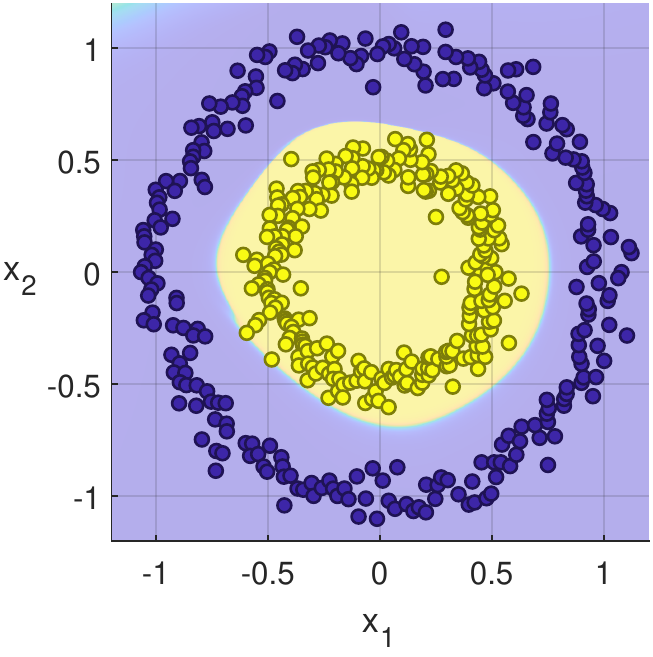}
\end{minipage}

\vskip2ex
{\footnotesize Trajectories:}

\noindent
\begin{minipage}[t]{0.329\linewidth}
\includegraphics[width=2in]{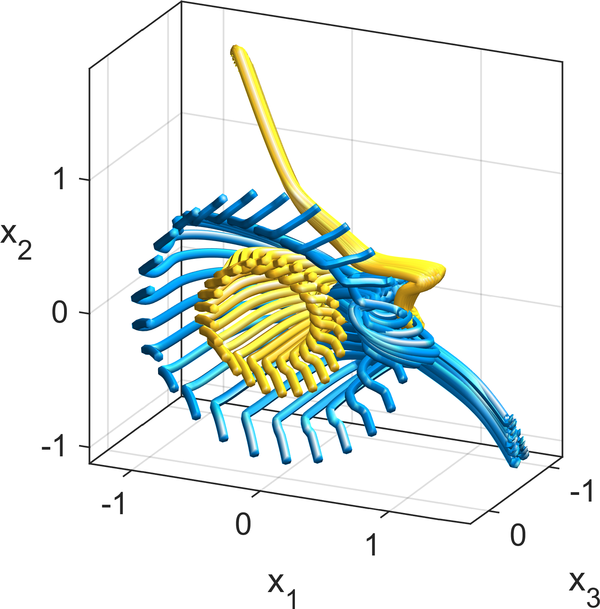}
\end{minipage}
\begin{minipage}[t]{0.329\linewidth}
\includegraphics[width=2in]{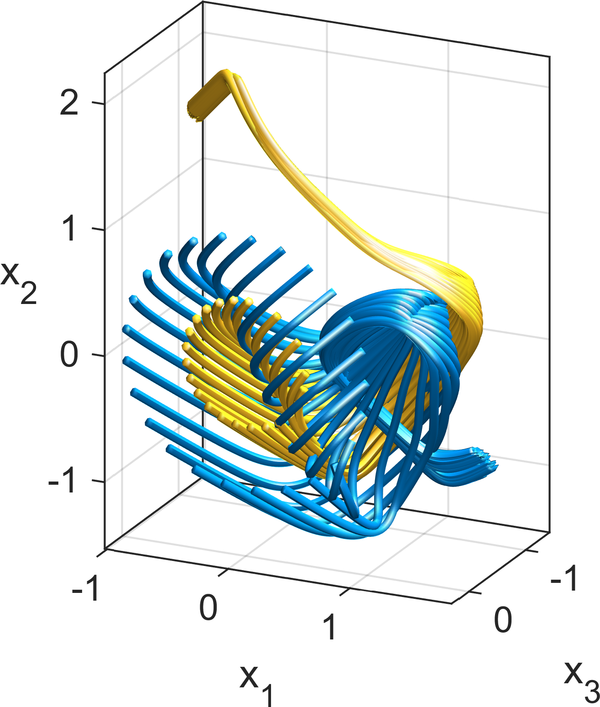}
\end{minipage}
\begin{minipage}[t]{0.329\linewidth}
\includegraphics[width=2in]{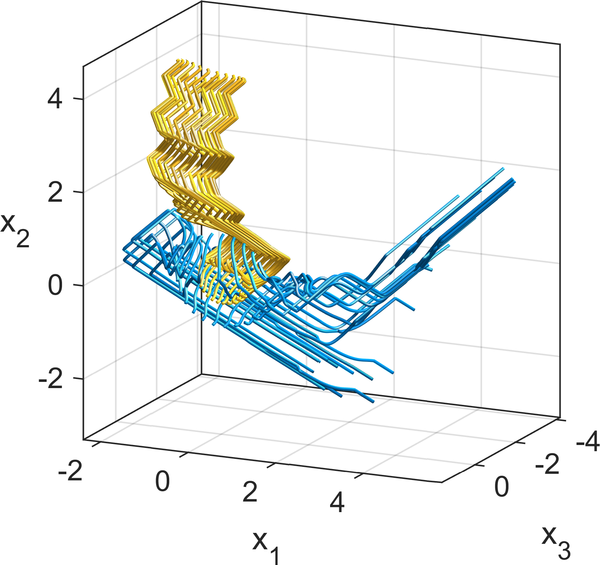}
\end{minipage}

\vskip1ex
\noindent
\begin{minipage}[t]{0.329\linewidth}
\begin{center}NCG with $L^2$-descent

$L^2$-penalization\end{center}
\end{minipage}
\begin{minipage}[t]{0.329\linewidth}
\begin{center}NCG with $W^{1,2}$-descent

$W^{1,2}$-penalization\end{center}
\end{minipage}
\begin{minipage}[t]{0.329\linewidth}
\begin{center}SGD\end{center}
\end{minipage}
\caption{The two circles dataset with augmentation. NCG depicted after epoch 3 in test run \#6.}
\label{fig:circles3D}
\end{figure}

The trajectory data in Figure~\ref{fig:circles3D} are 3-dimensional due to augmentation and hence the time component is not shown in the image. Therefore, the trajectories appear to be crossing. Nevertheless, the points do not collide as they reach the same position at different times $t$.

The effect of using the $W^{1,2}$-descent direction in NCG can be clearly seen in the graphs of the trainable parameters $W(t)$ and $b(t)$ in Figures \ref{fig:moons}--\ref{fig:circles3D}. Oscillations that are present when the $L^2$-direction is used are significantly attenuated. The corresponding effect can also be observed for the trajectories as they vary in a smoother fashion and the motion of the points is free of sharp turns due to diminished acceleration and twists.

Comparing Figures~\ref{fig:circles2D} and~\ref{fig:circles3D}, we see that augmentation has a profound effect for the two circles dataset (as has been pointed out above). In particular, the thin passage through the outer circle that connects the inner region to the outer parts of the plane disappears when the dimension is increased by augmentation. Namely, the additional dimension makes it possible for all the points of the inner region to follow a path towards the point $(0,1,0)$ without ever colliding with any of the points of the outer circle that are following a path towards $(1,0,0)$.

The obtained results compare well with the corresponding ones obtained with the SGD. Note that the weights obtained with SGD as well as the trajectories are oscillatory. Moreover, the continuous setting of NCG (before numerical discretization) corresponds to an infinite number of layers as opposed to the 250 layers employed in the SGD.

An additional benefit of the continuous NODE based model is that, due to well-posedness, it is rather easy to run the model backwards and forecast which region in the plane will map to the respective output class.

\section{Conclusions}
\label{sec:conclusions}
The inverse problem of finding depth-variable (time-dependent) parameters in a neural ordinary differential equation (NODE) has been investigated, corresponding to supervised reconstruction in a neural network with time continuous layers (an infinite number of layers). The NODE is considered in isolation, that is describing the full network where both the input and output are solely processed by the NODE and not by other layers outside this formulation, as opposed to previous works. The inverse problem is recast in the form of minimization of a cost functional, involving loss terms and penalty terms, subject to a general first order differential equation. A nonlinear conjugate gradient method (NCG) is developed for finding the minimum. The adjoint and sensitivity problems are derived, as well as the Fréchet derivative of the cost functional. The Sobolev gradient is also derived and incorporated to enhance smoothness of the reconstructed parameters. Mathematical properties such as parameter dependence of the differential equation are shown, and existence of a minimum to the cost functional is outlined. Numerical results are given for a particular NODE and two synthetic datasets, a two moon one and one corresponding to two full circles. In the examples, the sought parameters are a vectorization of a weight matrix and a bias vector, and supplied with an activation function. Comparison with standard  SGD based methods (discrete, not based on NODEs) show that NODEs in isolation perform well having additional advantages such as smoothness and stability, and being able to forecast forward in time as well as reversing to find the cause of an observed effect. As verified in the numerical examples, using the Sobolev gradient instead of the standard $L^2$ gradient has a marked effect on the results. It renders smoother parameters and trajectories, which in turn speeds up the solution of the requested NODEs in the NCG. There are recent results on training to also automatically find values of parameters for the penalty terms, see~\cite{Alberti}, that would be interesting to utilize. Investigating the sensitivity problem further has the potential to render qualitative statements about the behavior of error propagation in the learning process under noise.  The results presented are general covering a broad class of NODEs and cost functionals, and can therefore be applied to parameter identification in ODEs also outside the context of deep learning. It opens the possibility to turn other discrete optimization methods into continuous minimization subject to ODEs in the spirit of what has been initiated in~\cite{Ross}. The usefulness of learning an ODE for function approximation has been observed in works on time-series and medical forecasting~\cite{Fompeyrine}. Speeding up the calculations via the Sobolev gradient can further enhance the use of NODEs in applications.

\section*{Acknowledgments}
\addcontentsline{toc}{section}{Acknowledgments}
This work was supported by the LiU Cancer network at Linköping University, the research environment ELLIIT, and Linköping University Center for Industrial Information Technology (CENIIT).

\pdfbookmark[1]{References}{sec:refs}


\begin{thebibliography}{10}

\bibitem{Alberti}
{\sc G.~Alberti, E.~De~Vito, M.~Lassas, L.~Ratti, and M.~Santacesaria}, {\em
  Learning the optimal {T}ikhonov regularizer for inverse problems}, Advances
  in Neural Information Processing Systems, 34 (2021).

\bibitem{Alifanov}
{\sc O.~M. Alifanov}, {\em Inverse Heat Transfer Problems}, Springer, Berlin,
  1994.

\bibitem{Alosaimi}
{\sc M.~Alosaimi, D.~Lesnic, and B.~T. Johansson}, {\em Solution of the
  {C}auchy problem for the wave equation using iterative regularization},
  Inverse Problems in Science and Engineering,  (2021), pp.~1--15.

\bibitem{Andrei}
{\sc N.~Andrei}, {\em Nonlinear conjugate gradient methods for unconstrained
  optimization}, Springer, 2020.

\bibitem{arridge}
{\sc S.~Arridge, M.~de~Hoop, P.~Maass, O.~{\"O}ktem, C.~Sch{\"o}nlieb, and
  M.~Unser}, {\em Deep learning and inverse problems}, Snapshots of Modern
  Mathematics from {O}berwolfach,  (2019), pp.~1--13.

\bibitem{AvelinNystrom}
{\sc B.~Avelin and K.~Nystr{\"o}m}, {\em Neural {ODE}s as the deep limit of
  {R}es{N}ets with constant weights}, Analysis and Applications, 19 (2021),
  pp.~397--437.

\bibitem{Baravdish}
{\sc G.~Baravdish, B.~T. Johansson, W.~Ssebunjo, and O.~Svensson}, {\em
  Identifying the response of radiation therapy for brain tumors}, submitted to
  IMA J. Appl. Math.,  (2021).

\bibitem{benning}
{\sc M.~Benning, E.~Celledoni, M.~J. Ehrhardt, B.~Owren, and C.-B.
  Sch{\"o}nlieb}, {\em Deep learning as optimal control problems},
  IFAC-PapersOnLine, 54 (2021), pp.~620--623.

\bibitem{Brown}
{\sc A.~A. Brown and M.~C. Bartholomew-Biggs}, {\em Some effective methods for
  unconstrained optimization based on the solution of systems of ordinary
  differential equations}, Journal of Optimization Theory and Applications, 62
  (1989), pp.~211--224.

\bibitem{cao1}
{\sc K.~Cao and D.~Lesnic}, {\em Reconstruction of the perfusion coefficient
  from temperature measurements using the conjugate gradient method},
  International Journal of Computer Mathematics, 95 (2018), pp.~797--814.

\bibitem{cao2}
{\sc K.~Cao and D.~Lesnic}, {\em Reconstruction of the space-dependent
  perfusion coefficient from final time or time-average temperature
  measurements}, Journal of computational and applied mathematics, 337 (2018),
  pp.~150--165.

\bibitem{cao3}
{\sc K.~Cao, D.~Lesnic, and J.~Liu}, {\em Simultaneous reconstruction of
  space-dependent heat transfer coefficients and initial temperature}, Journal
  of Computational and Applied Mathematics, 375 (2020), p.~112800.

\bibitem{chen}
{\sc R.~T.~Q. Chen, Y.~Rubanova, J.~Bettencourt, and D.~Duvenaud}, {\em Neural
  ordinary differential equations}, in Proceedings of the 32nd International
  Conference on Neural Information Processing Systems, S.~Bengio, H.~M.
  Wallach, H.~Larochelle, K.~Grauman, and N.~Cesa-Bianchi, eds., Curran
  Associates Inc, Red Hook, NY, USA, 2018, pp.~6572--6583.

\bibitem{Dupont}
{\sc E.~Dupont, A.~Doucet, and Y.~W. Teh}, {\em Augmented neural {ODE}s}, in
  Advances in Neural Information Processing Systems 32, H.~Wallach,
  H.~Larochelle, A.~Beygelzimer, F.~d'Alch\'{e} Buc, E.~Fox, and R.~Garnet,
  eds., Curran Associates Inc, USA, 2019, pp.~3140--3150.

\bibitem{Zuazua}
{\sc C.~Esteve, B.~Geshkovski, D.~Pighin, and E.~Zuazua}, {\em Large-time
  asymptotics in deep learning}, arXiv preprint arXiv:2008.02491,  (2020).

\bibitem{Fompeyrine}
{\sc D.~Fompeyrine, E.~S. Vorm, N.~Ricka, F.~Rose, and G.~Pellegrin}, {\em
  Enhancing human-machine teaming for medical prognosis through neural ordinary
  differential equations {\rm(}{NODE}s{\rm)}}, Human-Intelligent Systems
  Integration, 3 (2021), pp.~263--275.

\bibitem{Hao}
{\sc D.~N. H{\`a}o, P.~X. Thanh, D.~Lesnic, and B.~T. Johansson}, {\em A
  boundary element method for a multi-dimensional inverse heat conduction
  problem}, International Journal of Computer Mathematics, 89 (2012),
  pp.~1540--1554.

\bibitem{hartman2002ordinary}
{\sc P.~Hartman}, {\em Ordinary Differential Equations: Second Edition},
  Classics in Applied Mathematics, Society for Industrial and Applied
  Mathematics, 2002, \url{https://books.google.se/books?id=v0z4ckZbuhMC}.

\bibitem{he2016deep}
{\sc K.~He, X.~Zhang, S.~Ren, and J.~Sun}, {\em Deep residual learning for
  image recognition}, in Proceedings of the IEEE conference on computer vision
  and pattern recognition, 2016, pp.~770--778,
  \url{https://doi.org/10.1109/CVPR.2016.90.}

\bibitem{higham}
{\sc C.~F. Higham and D.~J. Higham}, {\em Deep learning: An introduction for
  applied mathematicians}, SIAM Review, 61 (2019), pp.~860--891.

\bibitem{Hinton2012}
{\sc G.~Hinton, N.~Srivastava, and K.~Swersky}, {\em Neural networks for
  machine learning: Lecture 6a, overview of mini-batch gradient descent}, 2012.

\bibitem{Hofmann}
{\sc B.~Hofmann and C.~Hofmann}, {\em The impact of the discrepancy principle
  on the {T}ikhonov-regularized solutions with oversmoothing penalties},
  Mathematics, 8 (2020), p.~331.

\bibitem{8099726}
{\sc G.~Huang, Z.~Liu, L.~Van Der~Maaten, and K.~Q. Weinberger}, {\em Densely
  connected convolutional networks}, in 2017 IEEE Conference on Computer Vision
  and Pattern Recognition (CVPR), 2017, pp.~2261--2269,
  \url{https://doi.org/10.1109/CVPR.2017.243}.

\bibitem{Jin}
{\sc B.~Jin and J.~Zou}, {\em Numerical estimation of the {R}obin coefficient
  in a stationary diffusion equation}, IMA Journal of Numerical Analysis, 30
  (2010), pp.~677--701.

\bibitem{Khan}
{\sc K.~A. Khan and P.~I. Barton}, {\em Generalized derivatives for solutions
  of parametric ordinary differential equations with non-differentiable
  right-hand sides}, Journal of Optimization Theory and Applications, 163
  (2014), pp.~355--386.

\bibitem{kunze}
{\sc H.~E. Kunze and E.~R. Vrscay}, {\em Solving inverse problems for ordinary
  differential equations using the {P}icard contraction mapping}, Inverse
  Problems, 15 (1999), pp.~745--770.

\bibitem{Li}
{\sc Q.~Li, T.~Lin, and Z.~Shen}, {\em Deep learning via dynamical systems: An
  approximation perspective}, arXiv preprint arXiv:1912.10382,  (2019).

\bibitem{llibre}
{\sc J.~Llibre and R.~Ram{\'\i}rez}, {\em Inverse problems in ordinary
  differential equations and applications}, vol.~313, Springer, 2016.

\bibitem{dissecting}
{\sc S.~Massaroli, M.~Poli, J.~Park, A.~Yamashita, and H.~Asama}, {\em
  Dissecting neural {ODE}s}, in Advances in Neural Information Processing
  Systems, H.~Larochelle, M.~Ranzato, R.~Hadsell, M.~F. Balcan, and H.~Lin,
  eds., vol.~33, Curran Associates, Inc., 2020, pp.~3952--3963,
  \url{https://proceedings.neurips.cc/paper/2020/file/293835c2cc75b585649498ee74b395f5-Paper.pdf}.

\bibitem{Neuberger}
{\sc J.~Neuberger}, {\em Sobolev gradients and differential equations},
  Springer Science \& Business Media, 2~ed., 2009.

\bibitem{Regan}
{\sc D.~O'Regan}, {\em Existence theory for nonlinear ordinary differential
  equations}, vol.~398, Springer Science \& Business Media, 1997.

\bibitem{Owoyele_Pal}
{\sc O.~Owoyele and P.~Pal}, {\em Chemnode: A neural ordinary differential
  equations approach for chemical kinetics solvers}, arXiv preprint
  arXiv:2101.04749,  (2020).

\bibitem{Ross}
{\sc I.~M. Ross}, {\em An optimal control theory for nonlinear optimization},
  Journal of Computational and Applied Mathematics, 354 (2019), pp.~39--51.

\bibitem{Schuster2012}
{\sc T.~Schuster, B.~Kaltenbacher, B.~Hofmann, and K.~S. Kazimierski}, {\em
  Regularization methods in {B}anach spaces}, vol.~10 of Radon Series on
  Computational and Applied Mathematics, Walter de Gruyter GmbH \& Co. KG,
  Berlin, 2012, \url{https://doi.org/10.1515/9783110255720},
  \url{http://dx.doi.org/10.1515/9783110255720}.

\bibitem{Tabuada}
{\sc P.~Tabuada and B.~Gharesifard}, {\em Universal approximation power of deep
  neural networks via nonlinear control theory}, arXiv preprint
  arXiv:2007.06007,  (2020).

\bibitem{Teshima}
{\sc T.~Teshima, K.~Tojo, M.~Ikeda, I.~Ishikawa, and K.~Oono}, {\em Universal
  approximation property of neural ordinary differential equations}, arXiv
  preprint arXiv:2012.02414,  (2020).

\bibitem{Ursescu}
{\sc C.~Ursescu}, {\em A differentiable dependence on the right-hand side of
  solutions of ordinary differential equations}, Annales Polonici Mathematici,
  31 (1975), pp.~191--195.

\bibitem{8100117}
{\sc S.~Xie, R.~Girshick, P.~Dollár, Z.~Tu, and K.~He}, {\em Aggregated
  residual transformations for deep neural networks}, in 2017 IEEE Conference
  on Computer Vision and Pattern Recognition (CVPR), 2017, pp.~5987--5995,
  \url{https://doi.org/10.1109/CVPR.2017.634}.

\end{thebibliography}
\end{document}